%% file: main.tex
\documentclass[english]{article} %
\usepackage{iclr2027_conference,times}

\usepackage{xcolor}
\usepackage{colortbl}
\usepackage{url}
\usepackage{graphicx}
\usepackage{wrapfig}
\usepackage{float}
\usepackage{amsmath,amsfonts,bm}
\usepackage{amsthm}
\usepackage{booktabs}
\usepackage{adjustbox}
\usepackage{multirow}
\usepackage{listings}
\usepackage{tcolorbox}
\usepackage{aliascnt}
\usepackage{hyperref}
\usepackage{cleveref}
\usepackage{tabularx}
\usepackage{enumitem}

\newcommand{\METHOD}{VRSD}

\newcommand{\B}[1]{\textbf{#1}}

\definecolor{ptopsdblue}{HTML}{EDF5FC}

\definecolor{promptorange}{HTML}{A65324}
\definecolor{promptgreenone}{HTML}{527F24}
\definecolor{promptgreentwo}{HTML}{146D43}
\definecolor{promptgreenthree}{HTML}{053E2E}
\tcbset{teacherpromptbox/.style={
    colback=white,
    boxrule=0.5pt, arc=4pt, outer arc=4pt,
    boxsep=0pt, left=8pt, right=8pt, top=5pt, bottom=5pt,
    before skip=7pt, after skip=7pt,
    fonttitle=\normalfont\bfseries\footnotesize,
    attach title to upper={\par\vspace{3pt}}
}}
\lstdefinestyle{teacherprompttext}{
    basicstyle=\normalfont\footnotesize,
    breaklines=true, breakatwhitespace=true, breakindent=0pt,
    columns=fullflexible, keepspaces=true, frame=none,
    xleftmargin=0pt, xrightmargin=0pt, aboveskip=0pt, belowskip=0pt,
    escapeinside={(*@}{@*)}, literate={—}{{\textrm{---}}}1
}

\newtheorem{theorem}{Theorem}[section]
\newaliascnt{lemma}{theorem}

\aliascntresetthe{lemma}
\newaliascnt{proposition}{theorem}

\aliascntresetthe{proposition}
\newaliascnt{corollary}{theorem}
\newtheorem{corollary}[corollary]{Corollary}
\aliascntresetthe{corollary}

\crefname{section}{Section}{Sections}
\crefname{subsection}{Section}{Sections}
\crefname{subsubsection}{Section}{Sections}
\crefname{paragraph}{Section}{Sections}
\crefname{subparagraph}{Section}{Sections}
\crefname{appendix}{Appendix}{Appendices}
\crefname{subappendix}{Appendix}{Appendices}
\crefname{subsubappendix}{Appendix}{Appendices}
\crefname{subsubsubappendix}{Appendix}{Appendices}
\crefname{subsubsubsubappendix}{Appendix}{Appendices}
\crefname{table}{Table}{Tables}
\crefname{figure}{Figure}{Figures}
\crefname{equation}{Eq.}{Eqs.}
\crefname{subequation}{Eq.}{Eqs.}
\creflabelformat{equation}{(#2#1#3)}
\creflabelformat{subequation}{(#2#1#3)}
\crefname{theorem}{Theorem}{Theorems}
\crefname{lemma}{Lemma}{Lemmas}
\crefname{proposition}{Proposition}{Propositions}
\crefname{corollary}{Corollary}{Corollaries}

\input{math_commands.tex}

\title{Self-Distillation with Verbal Regularization}

\author{%
Jiazheng Li\textsuperscript{\ensuremath{\dagger},1}, Yawei Li\textsuperscript{\ensuremath{\dagger},2,3}, Junru Lu\textsuperscript{4}, Zhenyi Shen\textsuperscript{1}, Di Yin\textsuperscript{4},\\
Lin Gui\textsuperscript{1}, Xing Sun\textsuperscript{4}, David R\"ugamer\textsuperscript{2,3}, Xinpeng Wang\textsuperscript{*,2}, Yulan He\textsuperscript{*,1}}

\definecolor{preprintblue}{HTML}{245A81}
\makeatletter
\newcommand{\preprinttitle}{%
  {\raggedright
  {\LARGE\bfseries\color{preprintblue}%
  \mbox{\@title}\par}
  \vspace{12pt}
  {\normalsize\bfseries\@author\par}
  \vspace{4pt}
  {\normalsize
  \begin{tabular}[t]{@{}l@{\quad}l@{\quad}l@{}}
  \textsuperscript{1}King's College London & \textsuperscript{2}LMU Munich &
  \textsuperscript{3}Munich Center for Machine Learning\\
  \textsuperscript{4}Tencent YouTu Lab & \multicolumn{2}{@{}l@{}}{}
  \end{tabular}\par}
  \vspace{4pt}
  {\small
  \texttt{\{jiazheng.li,zhenyi.shen,yulan.he\}@kcl.ac.uk}\\
  \texttt{\{yawei.li,david.ruegamer\}@lmu.de},\quad \texttt{xinpeng@cis.lmu.de}\\
  \texttt{\{junrulu,endymecyyin,winfredsun\}@tencent.com}\\
  \par}}
  \vspace{10pt}
  {\color{preprintblue!40}\hrule height 0.4pt}
  \vspace{12pt}}
\makeatother
\renewenvironment{abstract}{\normalsize\leftskip=16pt\rightskip=16pt\parfillskip=0pt plus 1fil\noindent\ignorespaces}{\par}

\iclrfinalcopy
\hypersetup{
    hidelinks,
    pdftitle={Self-Distillation with Verbal Regularization},
    pdfauthor={Jiazheng Li, Yawei Li, Junru Lu, Zhenyi Shen, Di Yin, Lin Gui, Xing Sun, David R\"ugamer, Xinpeng Wang, Yulan He}
}
\begin{document}
\raggedbottom %

\lhead{Preprint}
\thispagestyle{plain}
\begin{tcolorbox}[colback=ptopsdblue,colframe=ptopsdblue,
  boxrule=0pt,arc=10pt,boxsep=0pt,left=20pt,right=20pt,top=30pt,bottom=12pt,
  before skip=0pt,after skip=18pt]
\preprinttitle
\input{sections/0_abstract.tex}
\vspace{9pt}
{\raggedright\fontsize{9}{11}\selectfont
\textbf{Code:} \url{https://github.com/lijiazheng99/VRSD}\par}
\vspace{7pt}
\newcommand{\logoimg}[2]{\raisebox{-0.5\height}{\includegraphics[height=#1]{#2}}}
\noindent
\logoimg{16pt}{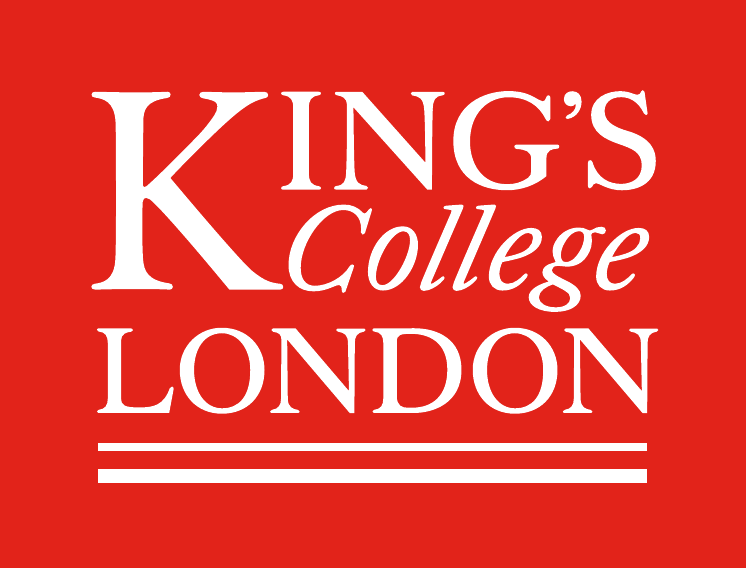}\hspace{7pt}%
\logoimg{16pt}{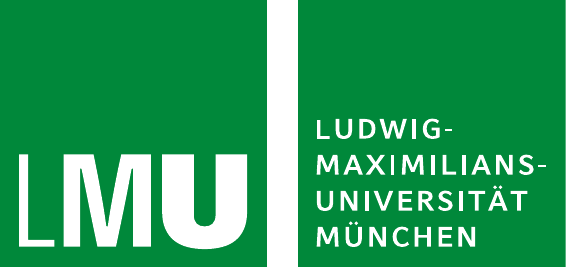}\hspace{7pt}%
\logoimg{15pt}{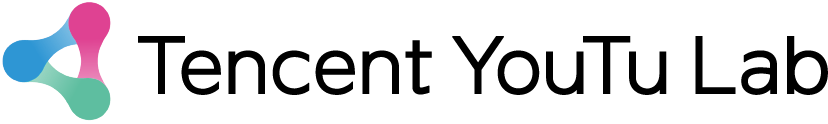}\hspace{7pt}%

\end{tcolorbox}
\begingroup
\renewcommand{\thefootnote}{\fnsymbol{footnote}}
\footnotetext[2]{Equal contribution.}
\footnotetext[1]{Corresponding authors: Xinpeng Wang and Yulan He.}
\endgroup
\input{sections/1_introduction.tex}
\input{sections/2_related_work.tex}
\input{sections/3_background.tex}
\input{sections/4_methodology.tex}
\input{sections/5_experiments.tex}

\input{sections/6_conclusion.tex}

\section*{Acknowledgment}
The work was supported in part by the UK Engineering and Physical Sciences Research Council through a Turing AI Fellowship (grant no. EP/V020579/1, EP/V020579/2) and the Prosperity Partnership scheme (grant no. UKRI566). 

\section*{AI use statement}
During the preparation of this manuscript, the authors used AI tools for code debugging and for improving the readability of text written by the authors. All conceptual development, theoretical analysis, the prompts used in the experiments, and interpretation of results were carried out independently by the authors. The authors have carefully reviewed and verified the content in the manuscript and take full responsibility for the integrity and accuracy of this submission.

\section*{Ethics statement}
Our contributions are algorithmic and foundational, focusing on the mechanics of on-policy self-distillation and teacher-side regularization in language models. All empirical evaluations were conducted on publicly released, standardized academic benchmarks (ASAP-SAS, SciKnowEval, Spider, and LiveCodeBench). The study does not involve human subjects, human annotations, or personally identifiable and sensitive data. All code generation and execution evaluations were conducted within isolated, secure sandbox environments. We foresee no immediate ethical concerns or direct societal harms arising from this work.

\section*{Reproducibility statement}

To ensure the full reproducibility of our theoretical and empirical findings, we provide comprehensive documentation throughout the paper and supplementary materials:
\begin{itemize}
    \item The mathematical derivation of the affine attenuation model for teacher mutual information is formalized in Section~\ref{sec:prompt-attenuation} and expanded in Appendix~\ref{app:proofs}.
    \item Exact training hyperparameters (including learning rates, batch sizes, rollout counts, context limits, and early-stopping criteria) across Qwen3-4B, Qwen3-8B, and OLMo-3-7B-Instruct are documented in Section~\ref{sec:setup} and summarized in Table~\ref{tab:hparams} (Appendix~\ref{app:hparams}).
    \item The complete verbatim system prompts for both the single-prompt and multi-prompt configurations across all four tasks are provided in Appendix~\ref{app:prompts}.
    \item The decoding setup and deterministic regex rules for student/teacher response-pattern extraction are detailed in Section~\ref{sec:qualitative} and Appendix~\ref{app:pattern-setup} (Table~\ref{tab:pattern-rules}).
    \item Training scripts, evaluation harnesses, and pattern-checking utilities are available in the code repository: \url{https://github.com/lijiazheng99/VRSD}.
\end{itemize}

\bibliography{iclr2027_conference}
\bibliographystyle{iclr2027_conference}
\newpage

\appendix
\crefalias{section}{appendix}
\crefalias{subsection}{subappendix}
\crefalias{subsubsection}{subsubappendix}
\crefalias{paragraph}{subsubsubappendix}
\crefalias{subparagraph}{subsubsubsubappendix}
\input{sections/7_appendix.tex}

\end{document}

%% file: math_commands.tex
\newcommand{\thetav}{\bm{\theta}}
\newcommand{\stopgrad}{\operatorname{sg}}
\newcommand{\PMI}{\operatorname{PMI}}
\newcommand{\dd}{\mathrm{d}}
\newcommand{\Cov}{\mathrm{Cov}}

%% file: sections/0_abstract.tex
\begin{abstract}
On-policy self-distillation (OPSD) trains a large language model using supervision from a teacher instantiated from the same model as the student. The teacher is conditioned on privileged information (PI), such as verified solutions. The student, which does not observe PI, learns from the teacher’s next-token distributions along its own sampled responses. However, PI can bias supervision toward shortcut reasoning with limited verification or reconsideration, accompanied by a sharp drop in the entropy of the output distribution. To address this challenge, we introduce self-distillation with verbal regularization (VRSD), which uses system prompts to regulate how the teacher uses PI, thereby encouraging rethinking and exploration to counteract premature commitment. Theoretically, we model prompting as attenuating the influence of PI on the teacher distribution and derive conditions for entropy preservation. Across three models and four tasks, VRSD improves Pass@1 over OPSD in 11 of 12 model–task settings, by 3.9 percentage points on average, while generally maintaining higher token entropy.
\end{abstract}

%% file: sections/1_introduction.tex
\section{Introduction} \label{sec:intro}

Knowledge distillation transfers  capabilities from a teacher to a student language model.
In conventional off-policy distillation, the student is trained on fixed teacher or reference responses~\citep{kim-rush-2016-sequence,hsieh-etal-2023-distilling,ho-etal-2023-large}. Because these responses are generated independently of the student's current policy, their prefixes may differ from those the student encountered during its own generation~\citep{opd1}.
On-policy distillation (OPD)~\citep{opd1,opd2} reduces this mismatch by providing dense teacher supervision on the student's own sampled responses.
Its dependence on a separate teacher nevertheless adds training cost.
On-policy self-distillation (OPSD)~\citep{opsd} removes this dependency by using the same model as both student and teacher under different conditioning.
Given an input question, the student first generates a response without access to privileged information (PI).
The self-teacher then uses additional PI, such as a verified solution, to produce a next-token distribution at each position of the student response.
The student is trained to match these teacher distributions.

\begin{figure}[!t]
    \centering
    \includegraphics[width=\linewidth]{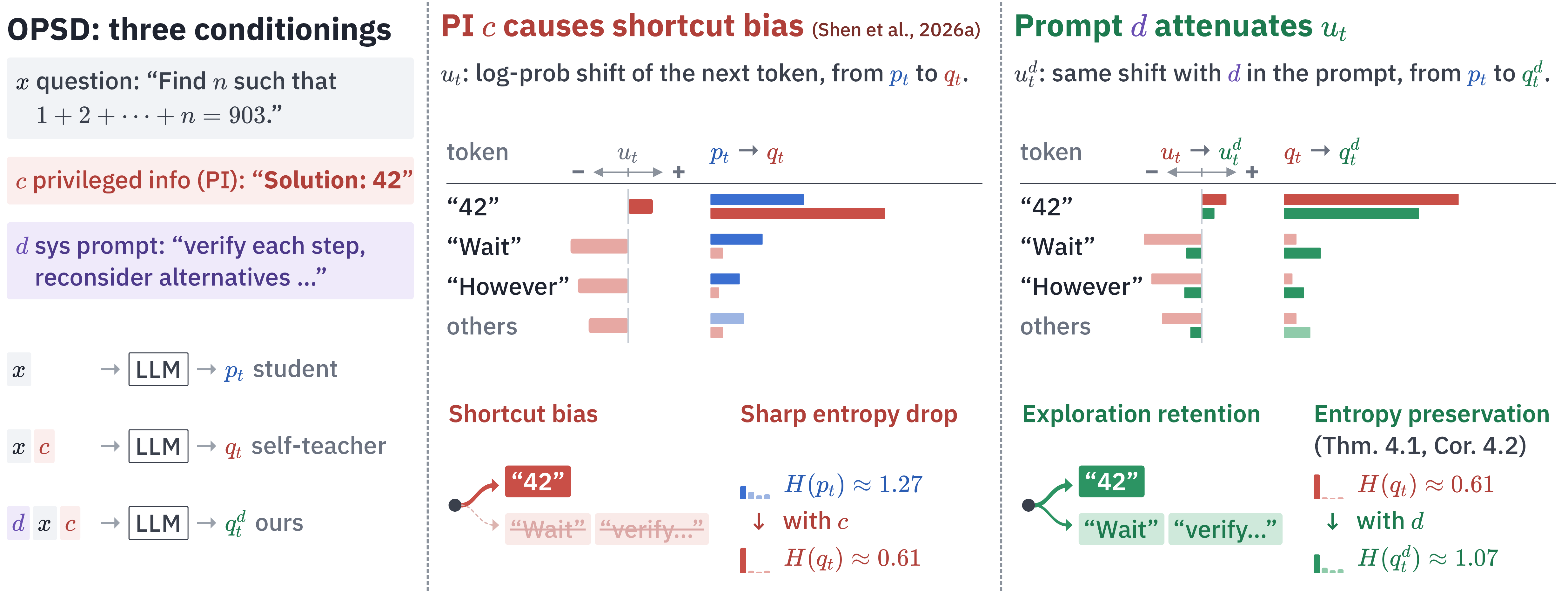}
    \vspace{-16pt}
    \caption{\METHOD{} counteracts OPSD's shortcut bias with a teacher-only prompt.
    \textbf{Left:} student $p_t$ (blue), self-teacher $q_t$ (red), and prompted teacher $q_t^d$ (green)
    under different conditioning.
    \textbf{Middle:} The log-probability shifts $u_t$ induced by PI $c$ can favor immediate commitment to the supplied answer (``42'')
    and suppress reconsideration (``Wait'', ``However''), concentrating $q_t$ and lowering its entropy $H(q_t)$.  Numerical values are illustrative. 
    \textbf{Right:} The system prompt $d$ encourages verification and reconsideration.
    By attenuating $u_t$, the prompt can restore probability to reconsideration in the teacher $q_t^d$ while retaining support for the answer.
    Under our affine model, \Cref{thm:teacher-entropy,cor:local-entropy-increase} show that
    attenuation preserves the teacher's entropy under mild conditions.}
    \label{fig:overview}
\end{figure}

However, previous work~\citep{kim2026doesselfdistillationsometimesdegrade,kaur2026rethinkingonpolicyselfdistillationthinking,shen2026anti} suggests that access to the PI can reinforce shortcut continuations while suppressing reconsideration.
This influence is described through log-probability shifts, which measure changes in the model's next-token log-probabilities when PI is added.
These shifts indicate which tokens become more or less likely at the same position in the student's response.
If these shifts reinforce tokens that the teacher already prefers, for example, tokens that immediately commit to the supplied answer (i.e., the ``shortcut'' tokens), they can make the teacher distribution more concentrated and reduce its entropy.
Distilling such supervision may encourage the student to inherit both shortcut reasoning behavior and the associated concentration of its predictive distribution, limiting opportunities to explore alternative reasoning paths.
The challenge is therefore to benefit from PI guidance while limiting its tendency to drive premature commitment.

To tackle this challenge, we introduce \textbf{self-distillation with verbal regularization (\METHOD{})}.
Rather than changing the model architecture, training objective, or PI itself, \METHOD{} modifies only the \emph{self-teacher's conditioning}.
Specifically, we introduce teacher-only system prompts %
that instruct the model to verify intermediate steps, consider alternatives, and reconsider the earlier reasoning before committing to an answer.
As illustrated in \Cref{fig:overview}, these prompts use the model's existing instruction-following capability to regulate how strongly PI shapes the teacher's token-level supervision.
The resulting method can therefore be incorporated directly into existing OPSD pipelines.
In our theoretical analysis, we model this regulation as \emph{affine attenuation} of the log-probability shifts induced by PI.
We characterize when prompting can counteract teacher-side entropy collapse by deriving an exact condition for entropy increase. Our contributions are as follows:
\begin{itemize}[leftmargin=1.2em]
	\item We introduce teacher-only verbal regularization for OPSD, using system prompts to regulate how the teacher uses privileged solutions without modifying the architecture or distillation objective.
 	\item We introduce an affine attenuation model of PI-induced log-probability shifts to explain when prompting counteracts teacher concentration reinforced by PI.
    \item \METHOD{} improves Pass@1 over OPSD in 11 of 12 model-task settings, while improving average Pass@1 and Worst@16 for all three models (\cref{tab:main}). Further analyses indicate that \METHOD{} mitigates over-concentration in output distributions (\cref{sec:entropy}). Causal interventions in \cref{sec:shortcut} indicate that, on average, VRSD students rely more on their reasoning chains to reach final answers compared to vanilla OPSD, indicating higher faithfulness in reasoning paths.
\end{itemize}

%% file: sections/2_related_work.tex
\section{Related work}

\textbf{OPD}\quad
Knowledge distillation \citep{hinton2015distillingknowledgeneuralnetwork,Gou_2021,xu2024surveyknowledgedistillationlarge} traditionally trains a student to match a teacher’s predictive distribution. 
For language models, distillation often uses token-level cross-entropy on hard teacher generated sequences \citep{kim-rush-2016-sequence,hsieh-etal-2023-distilling,ho-etal-2023-large}.
OPD~\citep{opd1,opd2,opd3,opd4} trains the student on its own responses by matching the teacher's next-token distributions at each step.
This reduces the distribution mismatch of off-policy distillation~\citep{opd1} and provides more dense supervision than GRPO-like methods~\citep{grpo,yu2025dapo,zheng2025group}.
Building on these advantages, multi-teacher OPD \citep{mopd1,mopd2,mopd3} has recently emerged as a practical paradigm for consolidating the complementary capabilities of multiple teachers into the student.

\textbf{OPSD}\quad
Within the OPSD framework~\citep{opsd,sdpo,sang2026crispcompressedreasoningiterative}, prior work has explored diverse sources of PI for teacher supervision.
This privileged context may include reference solutions~\citep{opsd}, successful~\citep{sdpo} or self-consistent responses~\citep{uopsd}, and feedback from environments~\citep{sdpo,zhang2026learningraresuccessrich,liu2026herohindsightenhancedreflectionenvironment} or humans~\citep{wang2026openclawrl}.
Related approaches extend this context with reflections~\citep{ropsd} or reusable skills~\citep{jiang2026distillskillsweightsprompts,wang2026skillsdskillconditionedselfdistillationmultiturn,yang2026opidonpolicyskilldistillation}.
More broadly, OPCD~\citep{ye2026onpolicycontextdistillationlanguage} distills behaviors specified by system prompts.
Its system-prompt setting supplies general instructions without instance-specific reference solutions.
However, in OPSD, conditioning the self-teacher on complete reference solutions can introduce shortcut bias~\citep{kim2026doesselfdistillationsometimesdegrade,kaur2026rethinkingonpolicyselfdistillationthinking,harne2026privilegedbiasedpiconditionedteachers}.
To mitigate this bias, Purified OPSD~\citep{purified} filters reference-driven shortcuts from training targets by comparing teacher predictions with and without the question.
AntiSD~\citep{shen2026anti} instead combines GRPO with updates that discourage matching the privileged teacher.
A complementary line of work modifies the PI supplied to the teacher: PS-OPSD~\citep{psopsd} replaces solutions with structured guidance, while ATESD~\citep{atesd} adaptively controls how much reference reasoning is revealed.
In this work, \METHOD{} instead regularizes how the teacher uses PI: it retains complete reference solutions and uses teacher-only system prompts to regulate their influence, aiming to mitigate shortcut bias and preserve the student’s exploration ability. 

%% file: sections/3_background.tex
\section{Background of on-policy self-distillation}\label{sec:background}

Let $\mathcal V$ denote the vocabulary and
$\mathcal V^*=\bigcup_{\ell > 0}\mathcal V^\ell$ the set of all finite token sequences.
Given an input $x\in\mathcal V^*$, an LLM $\pi_{\thetav}$ parametrized by $\thetav \in \Theta$ assigns a probability to each response $y=(y_1,\ldots,y_T)\in\mathcal V^T$ through the factorization
$\pi_{\thetav}(y\mid x)=\prod_{t=1}^{T}\pi_{\thetav}(y_t\mid x,y_{<t})$.

\textbf{OPSD}\quad
To formalize OPSD, let $\mathcal D$ be a training distribution over pairs
$(x,c)\in\mathcal V^*\times\mathcal V^*$.
Here, $c$ denotes privileged information (PI) that is available during training but not at
inference, such as a verified answer or a chain-of-thought extracted from a correct
trajectory.
Given $(x,c)\sim\mathcal D$, the model first acts as the student and samples a response
$y\sim\pi_{\thetav}(\cdot\mid x)$ without access to $c$.
At each resulting prefix $y_{<t}$, the student next-token distribution is
$p_t(\cdot):=\pi_{\thetav}(\cdot\mid x,y_{<t})$.
Evaluating the same model at the same prefix with PI gives the \emph{self-teacher} distribution
$q_t(\cdot):=\pi_{\thetav}(\cdot\mid x,c,y_{<t})$.
Thus, $p_t$ and $q_t$ share the parameters $\thetav$ and the prefix $y_{<t}$ but differ in
their access to $c$.
This conditioning asymmetry allows the same model to serve as its own teacher.

\textbf{OPSD with generalized Jensen-Shannon divergence}\quad
At each prefix $y_{<t}$, OPSD transfers the PI-informed teacher distribution to the
student using normalized generalized Jensen-Shannon divergence (JSD).
For a mixture weight $\alpha\in(0,1)$, let
$m_t=\alpha q_t+(1-\alpha)p_t$ denote the convex mixture of the self-teacher and student
distributions.
Their divergence is
\begin{equation}
    \mathcal{J}_\alpha(q_t\|p_t)
    =
    \frac{
        \alpha D_{\mathrm{KL}}(q_t\|m_t)
        +(1-\alpha)D_{\mathrm{KL}}(p_t\|m_t)
    }{
        \alpha(1-\alpha)
    },
    \label{eq:normalized-jsd}
\end{equation}
where $D_{\mathrm{KL}}(\cdot\|\cdot)$ denotes the Kullback-Leibler (KL) divergence.
Using $\mathcal{J}_\alpha$, OPSD minimizes
\begin{equation}
    \mathcal L_{\mathrm{OPSD}}(\thetav)
    =
    \mathbb E_{(x,c)\sim\mathcal D}
    \mathbb E_{y\sim\pi_{\thetav}(\cdot\mid x)}
    \left[
        \frac{1}{T}\sum_{t=1}^{T}
        \mathcal{J}_\alpha\!\left(\stopgrad[q_t]\|p_t\right)
    \right],
    \label{eq:bare-opsd-loss}
\end{equation}
where $\operatorname{sg}[\cdot]$ denotes the stop-gradient operator.
It keeps the self-teacher distribution fixed during differentiation, so each term updates
the student distribution toward the PI-informed teacher.

%% file: sections/4_methodology.tex
\section{Methodology}
\subsection{Regularizing premature teacher commitment}
OPSD uses PI to provide supervision signals.
We first examine how PI can bias the self-teacher toward premature commitment, and
how teacher-only system prompting can regulate this effect.

\textbf{PI can reinforce premature commitment}\quad
At a student-generated prefix $y_{<t}$, conditioning on PI changes the self-teacher's relative preferences
over possible next tokens.
Following the conditional pointwise mutual information (PMI) interpretation of AntiSD~\citep{shen2026anti}, we quantify the PI-induced log-probability shift for token
$v\in\mathcal V$ as
\begin{equation}
    u_t(v):=\log\frac{q_t(v)}{p_t(v)}
    =\PMI(v;c\mid x,y_{<t}).
\end{equation}
When $c$ contains PI such as verified solution, a positive $u_t(v)$ means that access to it makes
$v$ more likely.
Such support can guide reasoning, but it can also favor direct commitment to the supplied
answer before the current reasoning has justified it.
If PI primarily reinforces already-preferred conclusion tokens while suppressing
tokens associated with verification or reconsideration, the resulting teacher distribution becomes more concentrated (i.e., low entropy) and assigns little
probability mass to alternatives to premature commitment.

Such concentration can propagate to the student because it is trained to match
the teacher's distribution.
This preference transfer is reflected in the JSD gradient:
$-\nabla_{\thetav}\mathcal J_\alpha(\stopgrad[q_t]\|p_t)
=\mathbb E_{v\sim p_t}[A_t(u_t(v))\nabla_{\thetav}\log p_t(v)]$,
where $A_t(u)=\alpha^{-1}\log[(1-\alpha)+\alpha e^u]$.
Since $A_t(u)$ is strictly increasing in the PMI score $u$, tokens receiving larger positive PI-induced shifts (i.e., larger
PMI scores) also have larger coefficients in the distillation gradient.
Thus, when PI strongly favors premature commitment, the student is explicitly trained to increase the %
probability of those %
continuations. If the student approaches a highly concentrated teacher distribution, competing continuations become less likely under the student's own policy.
As subsequent training uses student-generated rollouts, these alternatives become less
likely to be sampled, limiting opportunities to explore different reasoning paths.
This link between teacher concentration and on-policy exploration motivates preserving probabilities on continuations that explore alternatives or reconsider intermediate steps.

\textbf{Regularizing teacher with system prompts}\quad
We use the self-teacher's instruction-following capability to regulate how it uses PI. Specifically, we introduce a
training-only system prompt $d$ that instructs the teacher to verify intermediate steps, or consider alternative reasoning paths before committing to an answer (see \cref{app:teachercont} for a behavioral analysis of the prompts' regularizing effects on teacher reasoning).
Formally, at the same prefix and with the same PI, the prompted teacher distribution is
$q_t^d(v)=\pi_{\thetav}(v\mid d,x,c,y_{<t})$ for $v\in\mathcal V$.
The student, by contrast, continues to generate without access to either the system prompt $d$ or the PI $c$.
Training replaces the original teacher distribution $q_t$ with the prompted distribution $q_t^d$ in \cref{eq:bare-opsd-loss}, giving the per-prefix loss
$\mathcal J_\alpha(\stopgrad[q_t^d]\|p_t)$, while leaving the JSD objective unchanged.
The prompt introduces no additional learned module or regularization loss; it is absent
from the student's inference-time input.
To quantify the effect of this prompt regularization, we next model how $d$ changes the teacher's
token preferences.

\textbf{An affine model of PMI attenuation}\quad
Using the same student distribution $p_t$ as reference, define the prompted PMI
$u_t^d(v):=\log[q_t^d(v)/p_t(v)]$.
Empirically, we find that the teacher-only system prompt reduced the PMI (i.e., the magnitude of PI-induced preference shifts) on average.
Such an attenuation effect can be modeled as
\begin{equation}
    u_t^d(v)=\gamma_tu_t(v)+\kappa_t,
    \qquad \gamma_t<1,\quad v\in\mathcal V.
    \label{eq:linear_attenuation}
\end{equation}
The affine coefficient $\gamma_t$ controls the relative scaling of PMI differences
across tokens, while $\kappa_t$ ensures normalization of $q_t^d$.
We estimate the factor $\gamma_t$
across models and tasks in \cref{sec:prompt-attenuation}.
\Cref{fig:gamma_estimate} shows mean estimates and 95\% bootstrap confidence intervals,
indicating statistically significant PMI attenuation. This supports the hypothesis that teacher prompting weakens PI-induced token preference shifts.
We next derive how this PMI attenuation changes the teacher distribution and, in particular, its entropy.

\begin{figure}[t]
    \centering
    \includegraphics[width=0.95\textwidth]{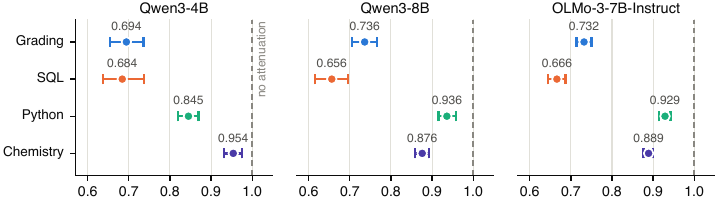}
    \vspace{-5mm}
    \caption{Estimated PMI attenuation induced by the teacher prompt $d$.
    Points show $\widehat\gamma$, obtained by averaging per-prefix estimates within each problem
    and then across problems; error bars show 95\% bootstrap confidence intervals (CIs).
    Dashed lines mark no attenuation.
    The CIs lie below this baseline, indicating that the system prompts attenuate PI-induced log-probability shifts on average.
    }
    \label{fig:gamma_estimate}
\end{figure}

\subsection{When PMI attenuation increases teacher entropy}
Experiments in \cref{sec:entropy} show that OPSD can produce \emph{low-entropy} models after training, potentially restricting exploration of alternative reasoning paths.
\METHOD{} mitigates this concentration by preserving the entropy.
Here, we provide a theoretical explanation for the observed entropy preservation by analyzing the teacher entropy change
$H(q_t^d)-H(q_t)$ under the affine PMI attenuation model, where $H(\cdot)$ denotes Shannon entropy. We first compute the prompted teacher distribution $q_t^d$ under the affine attenuation model in \cref{eq:linear_attenuation}, obtaining
\begin{equation}
    q_t^d(v)
    =\frac{p_t(v)\exp[\gamma_tu_t(v)]}
          {\sum_zp_t(z)\exp[\gamma_tu_t(z)]}
    =\frac{q_t(v)\exp[-(1-\gamma_t)u_t(v)]}
          {\sum_zq_t(z)\exp[-(1-\gamma_t)u_t(z)]}.
    \label{eq:attenuated-teacher-distribution}
\end{equation}
The derivation is provided in \cref{app:proof-teacher-distribution}. To analyze how this reweighting changes teacher entropy, introduce an intermediate
attenuation strength $\lambda\in[\gamma_t,1]$ and define
\begin{equation}
    q_{t,\lambda}(v)
    =
    \frac{q_t(v)\exp[(\lambda-1)u_t(v)]}
         {\sum_{z\in\mathcal V}q_t(z)\exp[(\lambda-1)u_t(z)]}.
    \label{eq:attenuation-path}
\end{equation}
This analytical path connects the bare and prompted teachers:
$q_{t,1}=q_t$ and, by \cref{eq:attenuated-teacher-distribution},
$q_{t,\gamma_t}=q_t^d$.
Next, \cref{thm:teacher-entropy} (with proof in \cref{app:proof-teacher-entropy}) gives an exact expression
for the entropy change induced by attenuation along this path.

\begin{theorem}
\label{thm:teacher-entropy}
Under the affine attenuation model in \cref{eq:linear_attenuation} with $\gamma_t<1$,
\begin{equation}
    H(q_t^d)-H(q_t)
    =\int_{\gamma_t}^{1}
    \Cov_{q_{t,\lambda}}\left(u_t,\log q_{t,\lambda}\right)\dd\lambda.
    \label{eq:teacher-entropy-condition}
\end{equation}
Here, $\Cov_{q_{t,\lambda}}(u_t,\log q_{t,\lambda})$ denotes the covariance between
the PMI score $u_t(v)$ and the log-probability $\log q_{t,\lambda}(v)$,
computed over $v\sim q_{t,\lambda}$.
\end{theorem}

\Cref{thm:teacher-entropy} evaluates covariance under the intermediate distributions $q_{t,\lambda}$
along the attenuation path.
For $\gamma_t$ close to $1$, a first-order expansion of the entropy change involves only
the covariance at $\lambda=1$, where $q_{t,1}=q_t$.
The following corollary therefore gives a local sufficient condition for entropy increase
in terms of the bare teacher distribution $q_t$.

\begin{corollary}
\label{cor:local-entropy-increase}
Under the affine attenuation model in \cref{eq:linear_attenuation}, if
$\Cov_{q_t}(u_t,\log q_t)>0$, then there exists $\delta_t>0$ such that
$H(q_t^d)>H(q_t)$ for every $\gamma_t\in(1-\delta_t,1)$.
\end{corollary}
\Cref{cor:local-entropy-increase} considers $\gamma_t$ sufficiently close to $1$. In this regime, teacher entropy increases when PMI scores and teacher log-probabilities are positively correlated under $q_t$, as proven in \cref{app:proof-local-entropy-increase}.
Motivated by AntiSD's token-level analysis, in \cref{app:covariance-test} we examine prefixes where PI substantially changes the sampled continuation, identified by $|u_t(y_t)|>1$. At these prefixes, we compute the covariance over the full vocabulary under $q_t$. \cref{tab:cov-strat} in \cref{app:covariance-test} reports mean covariance across various model-task settings, ranging from $0.43$ to $0.58$, with the covariance at $79\%-91\%$ of these prefixes being positive. These measurements support the covariance condition at prefixes with substantial PMI scores. 

%% file: sections/5_experiments.tex
\section{Experiments}\label{sec:experiments}

The experiments are presented in the following order. We first assess the task performance of VRSD in \cref{sec:main}. Next, in \cref{sec:entropy}, we examine student entropy to determine whether the entropy-preserving effect of teacher-side regularization carries over to the student. In \cref{sec:promptset}, we then investigate which instruction choices contribute to the performance gains. Since our ultimate goal is to encourage the student to learn better reasoning behaviors, we next examine whether it exhibits the intended reasoning patterns. We present this qualitative study in \cref{sec:qualitative}. Finally, we extend this behavioral analysis with a quantitative assessment of reasoning faithfulness. In \cref{sec:shortcut}, we test whether the student's answers causally depend on its step-by-step reasoning.

\subsection{Experimental setup}\label{sec:setup}

\textbf{Datasets}\quad We train and evaluate on four tasks with automatic verifiers: short-answer \emph{grading} on ASAP-SAS~\citep{hamner2012asapsas}, \emph{chemistry} on SciKnowEval~\citep{feng2024sciknoweval}, \emph{SQL} on the Spider development subset~\citep{yu-etal-2018-spider}, and \emph{Python} code generation on LiveCodeBench~\citep{jain2025livecodebench}. The PI $c$ is a verified solution to the same problem, obtained from a correct model rollout.

\textbf{Teacher system prompts}\quad We use two designs, \textcolor{promptorange}{\METHOD{} (single)} and \textcolor{promptgreentwo}{\METHOD{} (multi)}. In the single-prompt setting, all rollouts share a single composite task-level prompt $d$ that encourages reflection and systematic reasoning. In the multi-prompt setting, the rollouts of each input $x$ are split as evenly as possible across $K$ aspect-specific prompts $d_1,\ldots,d_K$, each targeting one aspect of the reasoning process; a rollout assigned to $d_k$ uses $q_t^{d_k}$ as its teacher distribution. The chemistry prompts ($K{=}3$) are shown below; all prompts and the value of $K$ for each task are given in \cref{app:prompts}.

\begin{tcolorbox}[teacherpromptbox,
    colframe=promptorange!50!white, coltitle=promptorange, title={\METHOD{} (single)}]
\begin{lstlisting}[style=teacherprompttext, basicstyle=\normalfont\footnotesize\color{promptorange}]
(*@$\boldsymbol{d}$@*)  You are a careful, reflective problem solver. Before committing to an answer, pause and reflect on your reasoning: re-examine each step for mistakes, question assumptions, check your arithmetic, and consider whether your earlier attempt went wrong and why. Only then give your final answer.
\end{lstlisting}
\end{tcolorbox}

\begin{tcolorbox}[teacherpromptbox,
    colframe=promptgreentwo!45!white, coltitle=promptgreentwo, title={\METHOD{} (multi)}]
\begin{lstlisting}[style=teacherprompttext, basicstyle=\normalfont\footnotesize\color{promptgreenone}]
(*@$\boldsymbol{d}_1$@*)  You are a chemistry expert who never trusts a first pick. Go through all four options and give the concrete reason each is right or wrong: e.g. the specific rule, mechanism, periodic trend, or calculation that decides it, not a gut feeling. After a provisional choice, deliberately reconsider: which of the other options is the classic trap for this concept, and what exact fact would make IT correct instead? If that case is stronger, backtrack and switch. Commit only once every option has been ruled in or out by an explicit chemical reason.
\end{lstlisting}
    \vspace{5pt}
\begin{lstlisting}[style=teacherprompttext, basicstyle=\normalfont\footnotesize\color{promptgreentwo}]
(*@$\boldsymbol{d}_2$@*)  You are a chemistry expert who reasons from mechanism, not memory. Before choosing, name the governing principle this question tests (reaction type, periodic trend, equilibrium, conservation, bonding, or structure-property rule) and derive what it predicts. Then match that derivation against each option. If your derivation and your first instinct disagree, do not default to instinct. Pause, find which step is wrong, and re-derive. Commit only when the principle and the chosen option genuinely agree.
\end{lstlisting}
    \vspace{5pt}
\begin{lstlisting}[style=teacherprompttext, basicstyle=\normalfont\footnotesize\color{promptgreenthree}]
(*@$\boldsymbol{d}_3$@*)  You are a chemistry expert who challenges your own answer. Make a provisional choice, then argue against it: what is the strongest case that a different option is correct, and what specific fact would you have had to overlook for your pick to be wrong? Separately, name the distractor most easily confused with your answer and say exactly what distinguishes them. Weigh both countercases, and switch whenever the counter rests on firmer chemistry. Backtrack explicitly the moment a step in your reasoning fails to hold.
\end{lstlisting}
\end{tcolorbox}

\textbf{Models and training}\quad We fine-tune Qwen3-4B~\citep{opd3}, Qwen3-8B, and OLMo-3-7B-Instruct~\citep{olmo2025olmo}. All runs use $8$ rollouts per prompt, $32$ prompts per batch, a prompt length of 2048 tokens, no KL regularization, and early stopping after four consecutive validations without improvement. Grading, chemistry, and SQL share one configuration up to response length, while Python follows the OPSD code recipe because it requires a longer output length~\citep{sdpo}. The task-specific training settings listed in \cref{tab:hparams} (\cref{app:hparams}) are held fixed across methods and models within each task.

\textbf{Evaluation and baselines}\quad We compare against the original pretrained model before task-specific fine-tuning (labeled \emph{Untrained}), \emph{OPSD}~\citep{opsd}, and \emph{AntiSD}~\citep{shen2026anti}. For AntiSD, we use the method-specific hyperparameters reported in the original paper. For each training run, we select the checkpoint with the highest validation Pass@1. We evaluate all models by sampling 16 responses per problem at temperature 0.6 and top-$p$ 0.95. We report Pass@1 (mean accuracy across the 16 sampled responses), Maj@16 (accuracy of the majority-vote answer), and Worst@16 (the fraction of problems for which all 16 responses are correct). At the selected checkpoint, we also report $H(p)$, the mean next-token entropy over the student's rollouts.

\subsection{Task performance}\label{sec:main}

\begin{table}[t]\centering\small\setlength{\tabcolsep}{3.4pt}
\caption{Task performance (\%) across three models and four tasks. P@1, M@16, and W@16 denote Pass@1, Maj@16, and Worst@16, respectively. ``Untrained'' denotes the model before task-specific training. ``Average'' denotes the mean performance across all four tasks. Blue rows highlight \METHOD{}; \textbf{bold} and \underline{underlined} values denote the best and second-best results, respectively.}
\label{tab:main}
\begin{adjustbox}{max width=\textwidth}\begin{tabular}{l ccc ccc ccc ccc ccc}
\toprule
& \multicolumn{3}{c}{\textbf{Grading (ASAP-SAS)}} & \multicolumn{3}{c}{\textbf{Chemistry (SciKnowEval)}} & \multicolumn{3}{c}{\textbf{SQL (Spider)}} & \multicolumn{3}{c}{\textbf{Python (LiveCodeBench)}} & \multicolumn{3}{c}{\textbf{Average}}\\
\cmidrule(lr){2-4}\cmidrule(lr){5-7}\cmidrule(lr){8-10}\cmidrule(lr){11-13}\cmidrule(lr){14-16}
Method & P@1$\uparrow$ & M@16$\uparrow$ & W@16$\uparrow$ & P@1$\uparrow$ & M@16$\uparrow$ & W@16$\uparrow$ & P@1$\uparrow$ & M@16$\uparrow$ & W@16$\uparrow$ & P@1$\uparrow$ & M@16$\uparrow$ & W@16$\uparrow$ & P@1$\uparrow$ & M@16$\uparrow$ & W@16$\uparrow$\\
\midrule
\multicolumn{16}{l}{\textit{Fine-tuned model: Qwen3-4B}}\\[1pt]
Untrained & 57.2 & 58.7 & 37.9 & 70.7 & 71.7 & 58.3 & 71.3 & 72.4 & 59.4 & 47.2 & 50.5 & 26.6 & 61.6 & 63.3 & 45.6\\
OPSD & 71.9 & 71.8 & 61.3 & 81.7 & 81.9 & 79.2 & 69.8 & 69.7 & 64.9 & 61.4 & 61.9 & 51.8 & 71.2 & 71.3 & 64.3\\
AntiSD & 78.3 & \underline{79.1} & 63.4 & 80.7 & 80.7 & 80.2 & 79.5 & \underline{80.6} & 71.2 & 53.3 & 55.7 & 35.5 & 73.0 & 74.0 & 62.6\\
\rowcolor{ptopsdblue}
\METHOD{} (single) & \underline{78.7} & 78.7 & \underline{73.8} & \B{86.1} & \B{86.3} & \B{81.8} & \underline{80.0} & \underline{80.6} & \underline{71.7} & \B{64.1} & \B{65.2} & \underline{54.5} & \underline{77.2} & \underline{77.7} & \underline{70.5}\\
\rowcolor{ptopsdblue}
\METHOD{} (multi) & \B{83.0} & \B{83.5} & \B{76.6} & \underline{85.8} & \underline{86.2} & \underline{81.3} & \B{81.0} & \B{81.7} & \B{74.4} & \underline{63.8} & \underline{64.7} & \B{54.8} & \B{78.4} & \B{79.0} & \B{71.8}\\
\midrule
\multicolumn{16}{l}{\textit{Fine-tuned model: OLMo-3-7B-Instruct}}\\[1pt]
Untrained & 52.0 & 54.6 & 19.8 & 51.2 & 65.2 & 22.9 & 64.7 & 66.5 & 39.4 & 40.5 & 43.8 & 14.8 & 52.1 & 57.5 & 24.2\\
OPSD & 76.3 & 77.2 & 62.4 & 81.9 & 82.5 & 74.3 & 74.1 & 75.4 & \underline{64.6} & \underline{57.3} & \underline{59.2} & 37.4 & 72.4 & 73.6 & 59.7\\
AntiSD & 67.0 & 70.1 & 39.6 & 77.5 & 79.0 & 66.2 & \underline{74.3} & \underline{75.6} & 60.4 & 45.7 & 46.3 & 23.5 & 66.1 & 67.8 & 47.4\\
\rowcolor{ptopsdblue}
\METHOD{} (single) & \underline{83.9} & \B{84.9} & \B{74.9} & \underline{83.5} & \underline{84.2} & \B{79.2} & 72.5 & 73.3 & 64.3 & 56.6 & 57.4 & \B{45.6} & \underline{74.1} & \underline{75.0} & \B{66.0}\\
\rowcolor{ptopsdblue}
\METHOD{} (multi) & \B{84.3} & \underline{84.8} & \underline{74.7} & \B{83.8} & \B{84.5} & \underline{78.1} & \B{75.1} & \B{75.8} & \B{64.7} & \B{57.6} & \B{60.7} & \underline{40.4} & \B{75.2} & \B{76.5} & \underline{64.5}\\
\midrule
\multicolumn{16}{l}{\textit{Fine-tuned model: Qwen3-8B}}\\[1pt]
Untrained & 60.6 & 61.0 & 41.4 & 72.1 & 73.3 & 58.3 & 72.2 & 73.3 & 60.2 & 51.7 & 55.7 & 31.2 & 64.2 & 65.8 & 47.8\\
OPSD & 75.8 & 76.6 & 63.3 & \underline{83.8} & \underline{84.1} & \B{78.9} & 78.5 & 78.2 & \underline{73.0} & \B{65.6} & \B{66.7} & \B{56.3} & 75.9 & 76.4 & \underline{67.9}\\
AntiSD & 69.7 & 69.9 & 60.9 & 72.1 & 73.3 & 57.8 & 78.6 & 79.0 & \B{73.4} & 59.2 & 62.7 & 39.9 & 69.9 & 71.2 & 58.0\\
\rowcolor{ptopsdblue}
\METHOD{} (single) & \underline{77.5} & \underline{78.7} & \underline{64.9} & 83.1 & 83.2 & 77.4 & \B{80.2} & \B{80.8} & 72.1 & 63.7 & \underline{64.5} & 53.2 & \underline{76.1} & \underline{76.8} & 66.9\\
\rowcolor{ptopsdblue}
\METHOD{} (multi) & \B{81.2} & \B{81.9} & \B{69.4} & \B{84.2} & \B{84.6} & \underline{78.7} & \underline{79.0} & \underline{79.7} & 69.9 & \underline{65.4} & \B{66.7} & \underline{55.8} & \B{77.5} & \B{78.2} & \B{68.5}\\
\bottomrule
\end{tabular}\end{adjustbox}
\end{table}

\textbf{A single VRSD teacher prompt improves OPSD, while multiple prompts yield further gains.}
As shown in \cref{tab:main}, both variants improve average Pass@1 over OPSD on every evaluated model. 
\METHOD{} (multi) improves average Pass@1 over OPSD by 7.2, 2.8, and 1.6 percentage points on Qwen3-4B, OLMo-3-7B-Instruct, and Qwen3-8B, respectively, yielding an average gain of approximately 3.9 percentage points. Furthermore, it achieves higher Pass@1 in 11 of 12 model--task settings. 
The benefits also extend beyond mean accuracy. \METHOD{} (multi) matches or exceeds OPSD in Maj@16 across all settings and improves average Worst@16 on every model. These results show that VRSD improves both average accuracy and the frequency with which problems are answered correctly across all sampled responses. Compared to the single-prompt VRSD variant, the multi-prompt variant achieves higher average Pass@1 and Maj@16 on every model.

\textbf{\METHOD{} remains effective even when OPSD underperforms the starting model.}
The benefits also extend to a setting where OPSD itself reduces accuracy.On Qwen3-4B SQL, OPSD's Pass@1 falls below that of the starting model (69.8\% versus 71.3\%), whereas both VRSD variants exceed it, with \METHOD{} (multi) reaching 81.0\%. In this setting, changing the self-teacher's conditioning turns a decline in student accuracy into an improvement over the starting model, demonstrating that VRSD can enable effective self-distillation even when OPSD does not yield an accuracy gain.

\subsection{Student entropy}\label{sec:entropy}
\begin{figure}[t]
    \centering
    \includegraphics[width=0.8\textwidth]{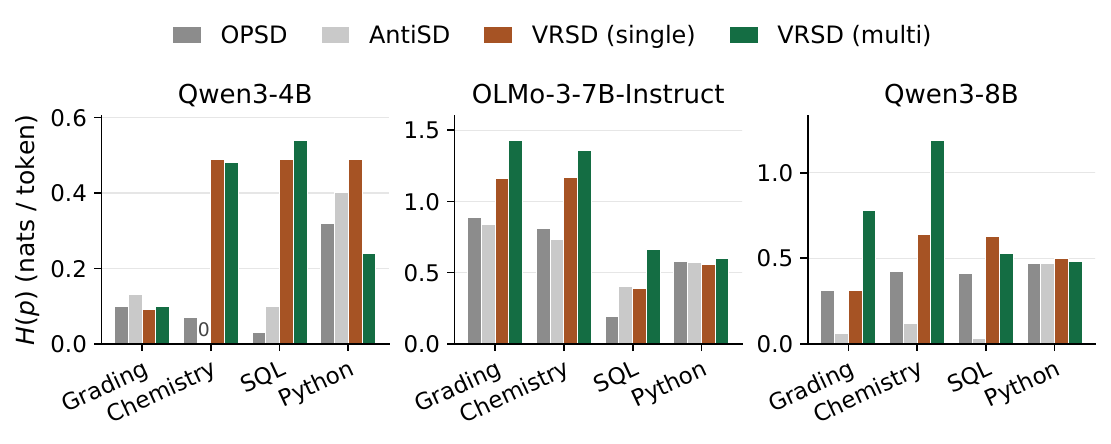}
    \vspace{-6mm}
    \caption{Average token entropy $H(p)$ of the distilled students. \METHOD{} generally yields higher entropy than OPSD; the detailed values and their ratios to OPSD are listed in \cref{tab:entropy} in \cref{app:entropy}.}
    \label{fig:entropy}
\end{figure}

\textbf{Teacher-side regularization mitigates over-concentration in the distilled student.}
Our theoretical analysis identifies conditions under which prompting counteracts concentration in the teacher distribution. The student is trained to match this distribution, and it is the model ultimately used to perform the tasks. We therefore measure the student's next-token entropy directly to assess whether the effect of teacher-side regularization carries over through distillation. \Cref{fig:entropy} shows that \METHOD{} produces less concentrated student distributions, with higher next-token entropy than OPSD in 19 of 24 comparisons across the two prompt variants. The clearest example is Qwen3-4B SQL, where OPSD produces a highly concentrated student distribution with $H(p)=0.03$, compared with 0.49 for \METHOD{} (single) and 0.54 for \METHOD{} (multi). For the multi-prompt variant, this reduced concentration accompanies a Pass@1 improvement from 69.8\% to 81.0\%. These results suggest that regularizing the teacher's supervision can help preserve the student's capacity to sample alternative continuations, leaving more scope for exploring different reasoning paths.

\subsection{Ablating the \METHOD{} (multi) design}\label{sec:promptset}

\begin{figure}[t]\centering
\includegraphics[width=\textwidth]{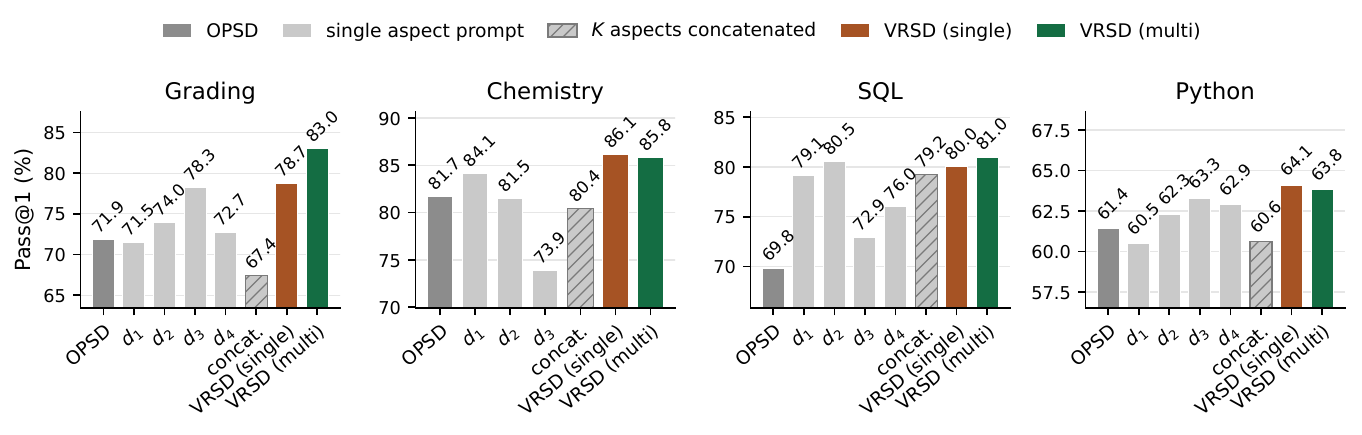}
\vspace{-8mm}
\caption{Pass@1 of Qwen3-4B under different teacher prompt designs. Each $d_k$ bar shows training with only the $k$-th teacher prompt from the \METHOD{} (multi) prompt set, used for all rollouts. Each of these prompts encourages a different approach to the problem, such as focusing on selected aspects. ``Concat.'' concatenates all $K$ prompts into one. \METHOD{} (single) uses a composite prompt $d$, whereas \METHOD{} (multi) assigns different prompts $d_k$ to different rollouts. The multi-prompt variant achieves higher Pass@1 than both the individual prompts and their concatenation across all four tasks.}
\label{fig:promptset}
\end{figure}

Does placing all instructions in one teacher prompt reproduce the benefits of distributing them across rollouts? We compare these choices on Qwen3-4B. Concatenation appends the original $d_k$ prompts unchanged, while \METHOD{} (single) uses a separately written composite prompt $d$ that presents an overall reasoning routine. Both settings use the same teacher prompt for every rollout. \METHOD{} (multi) assigns different $d_k$ to different rollouts. We also evaluate each $d_k$ alone. \Cref{fig:promptset} shows Pass@1, with full results in \cref{tab:promptset}.

\textbf{\METHOD{} (multi) improves over individual prompts and their concatenation.}
\METHOD{} (multi) exceeds concatenation and every individual prompt across all four tasks, improving over the strongest $d_k$ by 0.5 to 4.7 percentage points. Concatenation falls below OPSD on three tasks, while the composite prompt used by \METHOD{} (single) improves over both concatenation and OPSD on all four, showing that how instructions are combined within a prompt also matters. Compared with the composite prompt, the multi-prompt design further improves grading and SQL, with similar results on chemistry and Python. These results suggest that distributing different forms of guidance across rollouts can yield more effective supervision than applying an individual prompt alone or concatenating all prompts into one, leading to better student performance.

\subsection{Learning the reasoning patterns encouraged by teacher prompts}\label{sec:qualitative}

\begin{table}[t]\centering\small\setlength{\tabcolsep}{5pt}
\caption{Changes in student reasoning patterns relative to OPSD. We measure the percentage of student test responses exhibiting each pattern. Entries report the mean increase across three models, in percentage points. The count $(n/3)$ indicates how many models show an increase. Compared with OPSD, VRSD students more frequently exhibit the patterns encouraged by teacher prompts.}
\label{tab:patterns}
\begin{tabular}{llcc}
\toprule
\multirow{2}{*}{\textbf{Task}}
& \multirow{2}{*}{\shortstack{\textbf{Encouraged}\\\textbf{reasoning pattern}}}
& \multicolumn{2}{c}{\textbf{\METHOD{}} ($\Delta$ vs.\ OPSD)} \\
\cmidrule(lr){3-4}
& & \textbf{Single} & \textbf{Multi} \\
\midrule
\multirow{2}{*}{Grading} 
& Checks rubrics & +56.2 (3/3) & +41.5 (3/3) \\
& Quotes evidence & +68.2 (3/3) & +53.2 (3/3) \\
\midrule
\multirow{3}{*}{Chemistry} 
& Discusses all options & +23.5 (3/3) & +43.7 (3/3) \\
& Rules out options & +13.8 (3/3) & +22.6 (3/3) \\
& Checks conclusion & +73.6 (3/3) & +5.7 (2/3) \\
\midrule
\multirow{3}{*}{SQL} 
& Names schema tables & +36.4 (3/3) & +30.5 (3/3) \\
& Names join keys & +38.0 (3/3) & +29.4 (3/3) \\
& Plans query clauses & +53.4 (3/3) & +45.1 (3/3) \\
\midrule
\multirow{2}{*}{Python} 
& Discusses boundary cases & +19.0 (2/3) & +7.6 (2/3) \\
& Checks conclusion & +15.2 (3/3) & +5.9 (3/3) \\
\bottomrule
\end{tabular}
\end{table}

\textbf{Students adopt the reasoning patterns encouraged by teacher prompts.}
We measure how often student test responses exhibit the encouraged patterns, using the rules in \cref{app:behavior-heuristics}. Students generate these responses without teacher prompts or PI. \Cref{tab:patterns} shows positive average changes over OPSD for every reported pattern under both designs, with increases in most comparisons.

On chemistry, training with the reflective teacher prompt produces a larger increase in student conclusion checking than the multi-prompt design (+73.6 versus +5.7 percentage points over OPSD). The multi-prompt student shows larger increases in discussing and ruling out options, matching the prompts' emphasis on evaluating alternatives. Both designs also promote rubric checks and evidence quotations on grading, and explicit references to schema elements and query clauses on SQL. These correspondences suggest that teacher-side instructions can guide the reasoning patterns for the students.

\subsection{Evaluating reasoning faithfulness through causal interventions}\label{sec:shortcut}

\begin{table}[t]\centering\small\setlength{\tabcolsep}{4pt}
\caption{Evaluating reasoning faithfulness through causal interventions on Qwen3-4B. We retain different fractions of each student's reasoning chain and regenerate the final answer. \textbf{RG}$\uparrow$ (reasoning gain) is the Pass@1 difference between retaining the full reasoning and removing it, in percentage points. MR$\downarrow$ (misleading-reasoning rate) is the percentage of all evaluated answers that are correct without reasoning but incorrect when the full reasoning is retained. Compared with OPSD, VRSD achieves higher mean RG and lower mean MR, providing evidence of more faithful reasoning.}
\label{tab:shortcut}
\begin{adjustbox}{max width=\textwidth}\begin{tabular}{l cc cc cc cc cc}
\toprule
& \multicolumn{2}{c}{\textbf{Grading}} & \multicolumn{2}{c}{\textbf{Chemistry}} & \multicolumn{2}{c}{\textbf{SQL}} & \multicolumn{2}{c}{\textbf{Python}} & \multicolumn{2}{c}{\textbf{Average}} \\
\cmidrule(lr){2-3}\cmidrule(lr){4-5}\cmidrule(lr){6-7}\cmidrule(lr){8-9}\cmidrule(lr){10-11}
Student
& RG$\uparrow$ & MR$\downarrow$
& RG$\uparrow$ & MR$\downarrow$
& RG$\uparrow$ & MR$\downarrow$
& RG$\uparrow$ & MR$\downarrow$
& RG$\uparrow$ & MR$\downarrow$ \\
\midrule
Untrained & -0.8 & 13.4 & -0.3 & 5.9 & +2.3  & 4.9 & +2.5 & 6.7 & +0.9 & 7.7\\
OPSD & +5.2 & 2.5 & +4.2 & 1.9 & +2.3 & 4.9 & +3.2 & 2.5 & +3.7 & 3.0\\
\rowcolor{ptopsdblue}
\METHOD{} (single) & +11.9 & 1.8 & +1.1 & 1.6 & +5.5 & 1.9 & +4.7 & 1.1 & +5.8 & 1.6\\
\rowcolor{ptopsdblue}
\METHOD{} (multi) & +14.4 & 3.1 & +1.0 & 1.0 & +2.8 & 2.0 & +1.8 & 2.3 & +5.0 & 2.1\\
\bottomrule
\end{tabular}\end{adjustbox}
\end{table}

After examining which reasoning patterns students learn, we test whether their final answers depend on their reasoning. We evaluate Qwen3-4B using the truncation intervention of \citet{lanham2023measuring}. We sample student responses at the evaluation temperature, retain progressively larger fractions of each reasoning chain from none to the full chain, append the answer delimiter, and regenerate the final answer. As defined in \cref{tab:shortcut}, we assess students' reliance on explicit reasoning using reasoning gain (RG). 
We also report the misleading-reasoning rate (MR) to measure how often full reasoning overturns an otherwise correct answer.

\textbf{\METHOD{} students show stronger evidence of faithful reasoning.}
Retaining the full reasoning improves answer accuracy for both variants on all four tasks (\cref{tab:shortcut}). Averaged across tasks, RG reaches 5.8 percentage points for \METHOD{} (single) and 5.0 for \METHOD{} (multi), compared with 3.7 for OPSD. The effect is particularly clear on grading, where both variants more than double the RG of OPSD. The larger average RG is accompanied by lower MR than OPSD in seven of eight comparisons across the two variants and four tasks. These interventions suggest that \METHOD{} students rely more on their reasoning to reach correct answers than OPSD students do. Retaining their reasoning improves accuracy more on average and less often turns a correct answer into an incorrect one. Together, these findings provide evidence of improved reasoning faithfulness.

%% file: sections/6_conclusion.tex
\section{Conclusion}\label{sec:conclusion}

We introduced \METHOD{} to improve OPSD by regulating how the self-teacher uses privileged solutions. 
The central idea is to retain solution guidance while encouraging verification and reconsideration, shaping the supervision through teacher conditioning. 
Our analysis connects this intervention to attenuation of solution-induced log-probability shifts and identifies conditions under which it counteracts teacher concentration. 
Experiments complement this analysis with improved accuracy and answer consistency, generally higher student entropy, and evidence of improved reasoning faithfulness. 
These findings highlight teacher conditioning as a practical means of regulating privileged supervision. The model's existing instruction-following capability can help preserve alternatives to premature commitment and improve the learning signal it provides to itself.

Alongside these benefits, the current approach has several limitations that warrant further investigation. The current approach relies on manually curated, task-specific prompts. This requires human judgment about which reasoning behaviors to encourage, and the resulting prompts may not suit every problem or remain equally useful as the student's capabilities evolve during training. Future work could automate prompt selection and adaptation across problems and training stages. Beyond prompt design, verbal instructions control attenuation indirectly through the model's response to their wording. This motivates methods that estimate and adjust PI-induced log-probability shifts directly. Such methods could use our entropy analysis to guide where and how strongly to attenuate the influence of privileged solutions, enabling more adaptive control of teacher regularization.

%% file: sections/7_appendix.tex
\section{Estimating the PMI attenuation coefficient}
\label{sec:prompt-attenuation}

\subsection{Estimation algorithm}

Use the student trajectory $y\sim\pi_{\thetav}(\cdot\mid x)$ sampled by OPSD.
At each fixed prefix $y_{<t}$, evaluate the same model to obtain the background
distributions $p_t$ and $q_t$, together with the prompted PI-informed teacher
\begin{equation}
    q_t^d(v):=\pi_{\thetav}(v\mid d,x,c,y_{<t}),
    \qquad v\in\mathcal V,
    \label{eq:prompted-distributions}
\end{equation}
where $d$ denotes the reasoning-oriented teacher prompt.
Both teacher scores use the same unprompted, PI-free reference:
\begin{equation}
    u_t(v):=\log\frac{q_t(v)}{p_t(v)},
    \qquad
    u_t^d(v):=\log\frac{q_t^d(v)}{p_t(v)}.
    \label{eq:pi-scores}
\end{equation}
The score $u_t$ describes the bare PI-induced preference, while $u_t^d$ includes the
combined effect of PI and the prompt relative to that fixed reference.
This tests the net prompted-teacher shift along the original PMI direction, rather than
isolating the effect of PI conditional on $d$.
No prompt-only forward pass is required.

The affine approximation in \cref{eq:linear_attenuation} defines the analytical model
\begin{equation}
    \frac{q_t^d(v)}{p_t(v)}
    \propto
    \left(\frac{q_t(v)}{p_t(v)}\right)^{\gamma_t},
    \qquad \gamma_t<1,
    \label{eq:prompt-attenuation-model}
\end{equation}
where the attenuation factor $\gamma_t$ is shared across vocabulary tokens.
Taking logarithms gives $u_t^d(v)=\gamma_tu_t(v)+\kappa_t$, where $\kappa_t$ is
independent of $v$.

Let $\bar u_t=\sum_v q_t(v)u_t(v)$ and
$\bar u_t^d=\sum_v q_t(v)u_t^d(v)$.
Centering both scores under the same weights $q_t$ removes the normalization offset:
$u_t^d(v)-\bar u_t^d=\gamma_t[u_t(v)-\bar u_t]$ under the model.
For the observed full-vocabulary outputs, we extract the affine component through the
weighted least-squares slope:
\begin{equation}
    \widehat\gamma_t
    :=
    \frac{
        \sum_{v\in\mathcal V}q_t(v)
        \bigl(u_t(v)-\bar u_t\bigr)
        \bigl(u_t^d(v)-\bar u_t^d\bigr)
    }{
        \sum_{v\in\mathcal V}q_t(v)
        \bigl(u_t(v)-\bar u_t\bigr)^2
    }.
    \label{eq:gamma-estimator}
\end{equation}
We estimate \cref{eq:gamma-estimator} at each trajectory prefix and then average the
resulting values over positions within each input and across inputs.
A confidence interval whose upper endpoint is below $1$ provides evidence that the
mean affine-response coefficient is below $1$.

\renewcommand{\lstlistingname}{Algorithm}
\begin{lstlisting}[
    language=Python,
    basicstyle=\ttfamily\scriptsize,
    keywordstyle=\bfseries,
    commentstyle=\itshape,
    columns=fullflexible,
    keepspaces=true,
    showstringspaces=false,
    breaklines=true,
    frame=single,
    captionpos=t,
    caption={PyTorch-style estimation of the PMI attenuation coefficient.},
    label={alg:estimate-gamma}
]
@torch.no_grad()
def mean_attenuation(model, data, prompt, variance_tol=1e-6):
    gamma_x = []

    for x, c in data:
        # Sample one on-policy student rollout without PI or prompt.
        # The rollout may instead be pre-sampled and stored with x.
        y = sample_response(model, x=x)

        # Each call returns [T, vocabulary_size] log probabilities.
        log_p = teacher_forced_log_probs(model, x=x, response=y)
        log_q = teacher_forced_log_probs(model, x=x, pi=c, response=y)
        log_q_d = teacher_forced_log_probs(
            model, x=x, pi=c, response=y, system_prompt=prompt)

        gamma_t = []
        for t in evaluated_positions(y):
            q_t = log_q[t].exp()
            u_t = log_q[t] - log_p[t]
            # Both scores use the unprompted reference.
            u_t_d = log_q_d[t] - log_p[t]

            u_t = u_t - (q_t * u_t).sum()
            u_t_d = u_t_d - (q_t * u_t_d).sum()
            pi_variance = (q_t * u_t.square()).sum()

            # The attenuation slope is not identifiable here.
            if pi_variance <= variance_tol:
                continue

            gamma_t.append((q_t * u_t * u_t_d).sum()
                           / pi_variance)

        if gamma_t:
            gamma_x.append(torch.stack(gamma_t).mean())

    if not gamma_x:
        raise ValueError("No identifiable attenuation slopes.")

    gamma_x = torch.stack(gamma_x)
    return gamma_x.mean(), confidence_interval(gamma_x)
\end{lstlisting}

\subsection{Setup and results}\label{sec:gamma-results}

We estimate the PMI attenuation coefficient using \cref{eq:gamma-estimator} for all three models
before task-specific training to assess the effect of prompting on the initial policies.
For each problem, we use 16 validation rollouts sampled at temperature $0.6$ and top-$p$ $0.95$.
We pair a failed rollout $y$ with the verified solution $c$ from a successful rollout of the same problem,
following the OPSD training setup.
The teacher prompt $d$ is the single prompt used for the corresponding task, placed after $c$ as during training.

We obtain the full-vocabulary distributions $p_t$, $q_t$, and $q_t^d$ by teacher forcing on the same
rollout $y$ at 128 uniformly spaced prefixes per problem.
Problems without both a successful and a failed rollout are excluded; \cref{tab:gamma} reports the resulting problem counts.
We average the per-prefix estimates within each problem and then equally across problems, and compute
95\% bootstrap CIs by resampling problems.
The original task inputs for text-to-SQL and LiveCodeBench contain no system message;
for grading and chemistry, the existing system message is retained in its own role.
For Qwen3-4B on LiveCodeBench, we pool rollouts from four independent validations of the initial policy,
matching problems by exact prompt.
For Qwen3-8B and OLMo-3-7B-Instruct, we use a fresh validation of the initial policy for each task.

\begin{table}[h]\centering\small\setlength{\tabcolsep}{5pt}
\caption{Estimated PMI attenuation before task-specific training.
The mean coefficient $\widehat\gamma$ is obtained by averaging the per-prefix estimates in
\cref{eq:gamma-estimator} within each problem and then across problems, with 95\% bootstrap CIs over problems.}
\label{tab:gamma}
\begin{tabular}{l c c}
\toprule
Task & Problems & $\widehat\gamma$ [95\% CI]\\
\midrule
\multicolumn{3}{l}{\textit{Qwen3-4B}}\\[1pt]
Grading        & 100 & 0.694 [0.654, 0.736]\\
Text-to-SQL    &  46 & 0.684 [0.637, 0.737]\\
LiveCodeBench  &  38 & 0.845 [0.819, 0.870]\\
Chemistry      & 100 & 0.954 [0.932, 0.975]\\
\midrule
\multicolumn{3}{l}{\textit{Qwen3-8B}}\\[1pt]
Grading        & 100 & 0.736 [0.705, 0.766]\\
Text-to-SQL    &  42 & 0.656 [0.615, 0.696]\\
LiveCodeBench  &  36 & 0.936 [0.915, 0.958]\\
Chemistry      & 100 & 0.876 [0.859, 0.892]\\
\midrule
\multicolumn{3}{l}{\textit{OLMo-3-7B-Instruct}}\\[1pt]
Grading        & 100 & 0.732 [0.714, 0.750]\\
Text-to-SQL    &  96 & 0.666 [0.644, 0.687]\\
LiveCodeBench  &  53 & 0.929 [0.914, 0.942]\\
Chemistry      & 100 & 0.889 [0.877, 0.900]\\
\bottomrule
\end{tabular}
\end{table}

\Cref{tab:gamma} shows that the 95\% CI lies entirely within $(0,1)$ for every combination of model and task,
providing evidence of average PMI attenuation under teacher prompting.
Across all three models, attenuation is stronger on grading and text-to-SQL
($\widehat\gamma$ from $0.656$ to $0.736$) than on LiveCodeBench and chemistry ($0.845$ to $0.954$).
These results support the average attenuation effect underlying the affine model in \cref{eq:linear_attenuation}.

\section{Proofs of the theoretical results}
\label{app:proofs}

\subsection{Prompted teacher distribution}
\label{app:proof-teacher-distribution}

\begin{proof}[Derivation of \cref{eq:attenuated-teacher-distribution}]
Exponentiating \cref{eq:linear_attenuation} gives
$q_t^d(v)=p_t(v)\exp[\gamma_tu_t(v)+\kappa_t]$.
Normalization requires $\sum_vq_t^d(v)=1$, which determines the offset
$\kappa_t=-\log\sum_{z\in\mathcal V}p_t(z)\exp[\gamma_tu_t(z)]$.
Substituting this offset gives the first expression in
\cref{eq:attenuated-teacher-distribution}.
Using $q_t(v)=p_t(v)\exp[u_t(v)]$ yields the second expression.
\end{proof}

\subsection{Teacher entropy under PMI attenuation}
\label{app:proof-teacher-entropy}

\begin{proof}[Proof of \cref{thm:teacher-entropy}]
Let $Z_t(\lambda)=\sum_zq_t(z)\exp[(\lambda-1)u_t(z)]$ denote the normalizer
in \cref{eq:attenuation-path}.
Since $q_t$ and $u_t$ are fixed as $\lambda$ varies,
\begin{align*}
    \frac{\dd}{\dd\lambda}\log Z_t(\lambda)
    &=\frac{Z_t'(\lambda)}{Z_t(\lambda)}
      =\frac{\sum_zq_t(z)\exp[(\lambda-1)u_t(z)]u_t(z)}{Z_t(\lambda)}\\
    &=\sum_zq_{t,\lambda}(z)u_t(z).
\end{align*}
Taking logarithms of \cref{eq:attenuation-path} gives
$\log q_{t,\lambda}(v)=\log q_t(v)+(\lambda-1)u_t(v)-\log Z_t(\lambda)$.
Consequently,
\begin{align}
    \frac{\dd}{\dd\lambda}\log q_{t,\lambda}(v)
    &=u_t(v)-\sum_zq_{t,\lambda}(z)u_t(z),\nonumber\\
    \frac{\dd q_{t,\lambda}(v)}{\dd\lambda}
    &=q_{t,\lambda}(v)
      \left[u_t(v)-\sum_zq_{t,\lambda}(z)u_t(z)\right].
    \label{eq:app-attenuation-path-derivative}
\end{align}
Next, differentiate $H(q_{t,\lambda})=-\sum_vq_{t,\lambda}(v)\log q_{t,\lambda}(v)$
using the product rule:
\begin{align*}
    \frac{\dd H(q_{t,\lambda})}{\dd\lambda}
    &=-\sum_v\left[
        \frac{\dd q_{t,\lambda}(v)}{\dd\lambda}\log q_{t,\lambda}(v)
        +q_{t,\lambda}(v)\frac{\dd\log q_{t,\lambda}(v)}{\dd\lambda}
      \right]\\
    &=-\sum_v\frac{\dd q_{t,\lambda}(v)}{\dd\lambda}
      \left[\log q_{t,\lambda}(v)+1\right]\\
    &=-\sum_v\frac{\dd q_{t,\lambda}(v)}{\dd\lambda}\log q_{t,\lambda}(v).
\end{align*}
The last equality uses normalization:
$\sum_v\dd q_{t,\lambda}(v)/\dd\lambda
=\dd[\sum_vq_{t,\lambda}(v)]/\dd\lambda=0$.
Substituting \cref{eq:app-attenuation-path-derivative} and expanding the centered term gives
\begin{align}
    \frac{\dd H(q_{t,\lambda})}{\dd\lambda}
    &=-\sum_vq_{t,\lambda}(v)
      \left[u_t(v)-\sum_zq_{t,\lambda}(z)u_t(z)\right]
      \log q_{t,\lambda}(v)\nonumber\\
    &=-\sum_vq_{t,\lambda}(v)u_t(v)\log q_{t,\lambda}(v)\nonumber\\
    &\quad+\left[\sum_zq_{t,\lambda}(z)u_t(z)\right]
      \left[\sum_vq_{t,\lambda}(v)\log q_{t,\lambda}(v)\right]\nonumber\\
    &=-\Cov_{q_{t,\lambda}}\left(u_t,\log q_{t,\lambda}\right).
    \label{eq:app-teacher-entropy-derivative}
\end{align}
Finally, integrating from $\gamma_t$ to $1$ and using the two endpoints yields
\begin{equation}
    H(q_t^d)-H(q_t)
    =-\int_{\gamma_t}^{1}\frac{\dd H(q_{t,\lambda})}{\dd\lambda}\dd\lambda
    =\int_{\gamma_t}^{1}
    \Cov_{q_{t,\lambda}}\left(u_t,\log q_{t,\lambda}\right)\dd\lambda.
    \label{eq:app-teacher-entropy-change}
\end{equation}
\end{proof}

\subsection{Entropy increase under moderate attenuation}
\label{app:proof-local-entropy-increase}

\begin{proof}[Proof of \cref{cor:local-entropy-increase}]
Let $F_t(\lambda):=H(q_{t,\lambda})$ and $C_t:=\Cov_{q_t}(u_t,\log q_t)$.
Under the condition in \cref{cor:local-entropy-increase}, $C_t>0$.
On the finite vocabulary, \cref{eq:attenuation-path} defines a smooth path of distributions
with positive token probabilities, so $F_t$ has a continuous second derivative near $\lambda=1$.
Choose $r_t>0$ such that $F_t''$ is continuous on $[1-r_t,1]$.
Continuity on this closed interval guarantees a finite bound, so we can choose $M_t>0$
such that $|F_t''(\lambda)|\le 2M_t$ throughout the interval.

By \cref{eq:app-teacher-entropy-derivative} and $q_{t,1}=q_t$, we have $F_t'(1)=-C_t$.
For every $\gamma_t\in(1-r_t,1)$, Taylor's theorem and $q_{t,\gamma_t}=q_t^d$ give
\begin{equation}
    \begin{aligned}
        H(q_t^d)-H(q_t)
        &=F_t(\gamma_t)-F_t(1)\\
        &=(\gamma_t-1)F_t'(1)+R_t(\gamma_t)\\
        &=(1-\gamma_t)C_t+R_t(\gamma_t),
    \end{aligned}
    \label{eq:app-local-entropy-expansion}
\end{equation}
where $R_t(\gamma_t)=\tfrac12F_t''(\xi_t)(1-\gamma_t)^2$
for some $\xi_t\in(\gamma_t,1)$.
Since $\xi_t\in[1-r_t,1]$, the second-derivative bound yields
$|R_t(\gamma_t)|\le M_t(1-\gamma_t)^2$.

Choose $\delta_t=\min\{r_t,C_t/(2M_t)\}$.
For every $\gamma_t\in(1-\delta_t,1)$, the restriction $\delta_t\le r_t$
ensures that the remainder bound applies, while $\delta_t\le C_t/(2M_t)$ ensures
$M_t(1-\gamma_t)\le C_t/2$.
Consequently,
\begin{equation*}
    \begin{aligned}
        H(q_t^d)-H(q_t)
        &\ge (1-\gamma_t)C_t-|R_t(\gamma_t)|\\
        &\ge (1-\gamma_t)C_t-M_t(1-\gamma_t)^2\\
        &\ge \frac{C_t}{2}(1-\gamma_t)>0.
    \end{aligned}
\end{equation*}
\end{proof}

\section{Testing the covariance condition}
\label{app:covariance-test}

\subsection{Experimental setup}
\label{sec:cov-setup}

We examine the covariance condition in \cref{cor:local-entropy-increase} at prefixes where
PI changes the model's sampled continuation.
The study uses the initial untrained models.
For each input problem $x$, we pair a PI-free rollout $y$ with a verified solution $c$ from a
different successful rollout of the same problem, including both successful and failed
responses as $y$.
We obtain $p_t$ and $q_t$ by teacher forcing the same response under the two conditioning
contexts, using full-vocabulary log-probabilities at temperature $1$.

For \cref{fig:pmi-heatmap}, we bin sampled tokens by $p_t(y_t)$ and $q_t(y_t)$
and color each bin by its mean PMI $u_t(y_t)$, following AntiSD~\citep{shen2026anti}.
We group prefixes by sampled-token PMI magnitude, using $|u_t(y_t)|>1$ and
$|u_t(y_t)|\leq1$.
The threshold selects continuations whose probability changes by more than a factor of
$e$ under PI.
The \emph{Prefix share} column in \cref{tab:cov-strat} reports the percentage of all evaluated
prefixes with $|u_t(y_t)|>1$.

The sampled-token score determines the prefix group, while the covariance is computed
over the full vocabulary under $q_t$ in both groups.
With $\bar u_t=\sum_v q_t(v)u_t(v)$, we evaluate
\begin{equation}
    \Cov_{q_t}(u_t,\log q_t)
    =\sum_{v\in\mathcal V}q_t(v)
    \bigl(u_t(v)-\bar u_t\bigr)
    \bigl(\log q_t(v)+H(q_t)\bigr).
    \label{eq:covariance-diagnostic}
\end{equation}
For each group, we report the mean covariance and the percentage of prefixes with
positive covariance.

\subsection{Results and discussion}
\label{sec:cov-results}

\begin{figure}[h]\centering
\includegraphics[width=\textwidth]{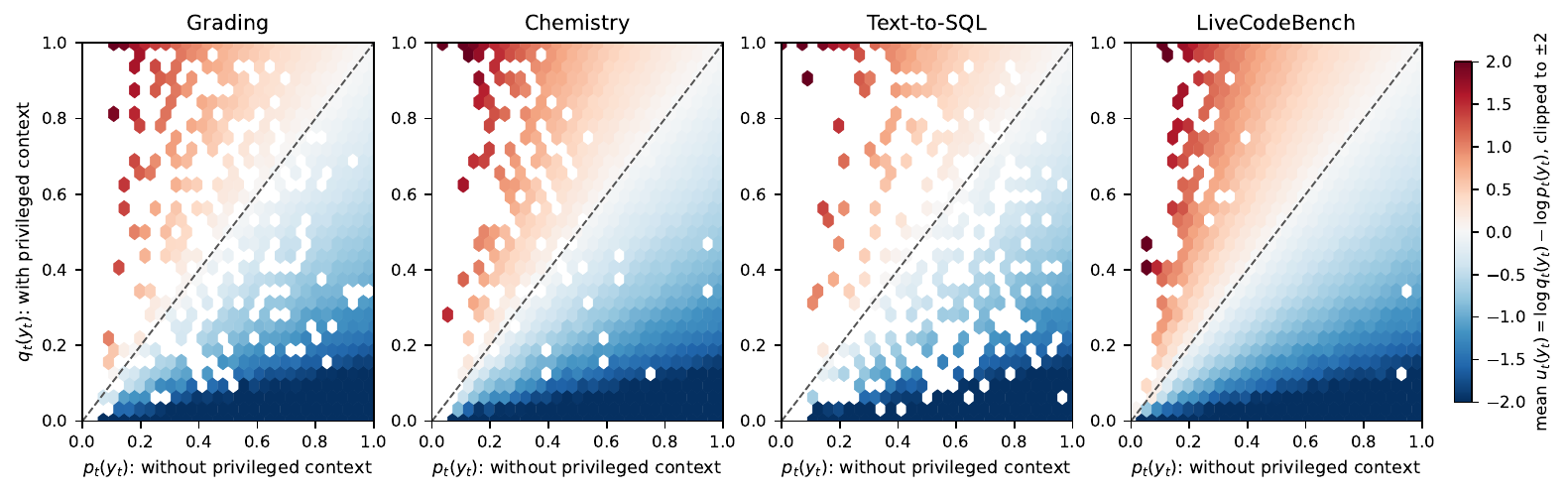}
\caption{Sampled-token PMI on the initial Qwen3-4B policy, following the visualization
of \citet{shen2026anti}.
Hexagonal bins group observations by $p_t(y_t)$ and $q_t(y_t)$; color shows the mean
$u_t(y_t)$, clipped to $[-2,2]$.
The dashed diagonal indicates unchanged probability.}
\label{fig:pmi-heatmap}
\end{figure}

\begin{table}[h]\centering\small\setlength{\tabcolsep}{3.5pt}
\caption{Teacher covariance grouped by sampled-token PMI magnitude.
The $|u_t(y_t)|>1$ group captures strong distillation signals despite its
small share of prefixes.
Prefix share gives the percentage of all evaluated prefixes with $|u_t(y_t)|>1$.
$\Cov$ denotes mean $\Cov_{q_t}(u_t,\log q_t)$.
$\Cov>0$ gives the percentage of prefixes with positive covariance within each group.}
\label{tab:cov-strat}
\begin{adjustbox}{max width=\textwidth}\begin{tabular}{l ccc cc}
\toprule
& \multicolumn{3}{c}{$|u_t(y_t)|>1$} & \multicolumn{2}{c}{$|u_t(y_t)|\leq1$}\\
\cmidrule(lr){2-4}\cmidrule(lr){5-6}
Task & Prefix share (\%) & $\Cov$ & $\Cov>0$ & $\Cov$ & $\Cov>0$\\
\midrule
\multicolumn{6}{l}{\textit{Qwen3-4B}}\\[1pt]
Grading        & 11.2\,\% & $+0.514$ & 90\,\% & $-0.083$ & 55\,\%\\
Text-to-SQL    & 6.4\,\% & $+0.449$ & 86\,\% & $-0.052$ & 61\,\%\\
LiveCodeBench  & 5.1\,\% & $+0.486$ & 87\,\% & $-0.049$ & 43\,\%\\
Chemistry      & 8.0\,\% & $+0.445$ & 90\,\% & $-0.048$ & 56\,\%\\
\midrule
\multicolumn{6}{l}{\textit{Qwen3-8B}}\\[1pt]
Grading        & 8.9\,\% & $+0.442$ & 89\,\% & $-0.052$ & 50\,\%\\
Text-to-SQL    & 7.2\,\% & $+0.435$ & 89\,\% & $-0.035$ & 70\,\%\\
LiveCodeBench  & 4.6\,\% & $+0.430$ & 84\,\% & $-0.048$ & 44\,\%\\
Chemistry      & 6.8\,\% & $+0.510$ & 91\,\% & $-0.021$ & 65\,\%\\
\midrule
\multicolumn{6}{l}{\textit{OLMo-3-7B-Instruct}}\\[1pt]
Grading        & 13.6\,\% & $+0.533$ & 79\,\% & $-0.108$ & 44\,\%\\
Text-to-SQL    & 9.8\,\% & $+0.501$ & 80\,\% & $-0.108$ & 46\,\%\\
LiveCodeBench  & 8.7\,\% & $+0.471$ & 81\,\% & $-0.056$ & 47\,\%\\
Chemistry      & 14.0\,\% & $+0.579$ & 82\,\% & $-0.101$ & 47\,\%\\
\bottomrule
\end{tabular}\end{adjustbox}
\end{table}

\Cref{fig:pmi-heatmap} illustrates how PI reinforces some sampled continuations and
suppresses others, consistent with the token-level pattern reported by AntiSD~\citep{shen2026anti}.
\Cref{tab:cov-strat} connects this sampled-token view to the full-vocabulary covariance
condition by grouping prefixes according to PMI magnitude.
PMI governs token-level distillation advantages, making the $|u_t(y_t)|>1$ group
particularly informative about strong supervision signals.
These prefixes account for approximately 5 to 14\,\% of evaluated positions,
and most of their sampled tokens have negative PMI.

Within this group, mean covariance is positive in all twelve model and task settings,
ranging from $0.43$ to $0.58$, and 79 to 91\,\% of prefixes have positive covariance.
Thus, where the sampled-token PMI has large magnitude, stronger PI support tends to
align with higher teacher log-probability across the vocabulary.
The complementary group has smaller, negative mean covariance, ranging from $-0.11$ to
$-0.02$.
Grouping by PMI magnitude therefore reveals where the covariance condition in
\cref{cor:local-entropy-increase} is most frequently satisfied.
Under the affine attenuation model, this positive alignment implies that sufficiently
small attenuation increases teacher entropy.

\section{Training configuration}\label{app:hparams}

\Cref{tab:hparams} lists the per-task training settings. Every run, on every model and for every method, uses $8$ rollouts per prompt, a batch size of 32, a prompt length of 2048 tokens, a sampling temperature of 1 during training, no KL regularization term, and early stopping after four consecutive validations without improvement. Evaluation always draws 16 samples per problem at temperature 0.6 and top-$p$ 0.95. Grading, chemistry, and SQL share one configuration. The Python task follows the OPSD code recipe \citep{sdpo}: the reverse KL ($\alpha{=}1$), a ten times smaller learning rate without warm-up, no entropy bonus, no advantage normalization, one optimizer update per rollout, and a longer response budget. Experiments on grading, chemistry, and SQL are run on two NVIDIA H100 (80\,GB) GPUs. Python experiments, which require longer outputs, are run on two NVIDIA H200 (141\,GB) GPUs.

\textbf{Acknowledgement}\quad We acknowledge the computational resources provided by King's Computational Research, Engineering and Technology Environment (CREATE)~\citep{kcl2026create}.

\begin{table}[H]\centering\small\setlength{\tabcolsep}{6pt}
\caption{Task-specific training settings. Settings not listed are shared by all runs (see text).}
\label{tab:hparams}
\begin{tabular}{l cccc}
\toprule
& Grading & Chemistry & SQL & Python\\
& (ASAP-SAS) & (SciKnowEval) & (Spider) & (LiveCodeBench)\\
\midrule
Divergence skew $\alpha$ & 0.5 & 0.5 & 0.5 & 1.0\\
Learning rate & $10^{-5}$ & $10^{-5}$ & $10^{-5}$ & $10^{-6}$\\
Warm-up steps & 10 & 10 & 10 & 0\\
Entropy bonus & $10^{-3}$ & $10^{-3}$ & $10^{-3}$ & 0\\
Advantage normalization & True & True & True & False \\
PPO mini-batch size & 32 & 32 & 32 & 1\\
Max.\ response length (tokens) & 2048 & 2048 & 2048 & 8192\\
Max.\ epochs & 4 & 4 & 4 & 30\\
Validation every (steps) & 20 & 20 & 20 & 5\\
\bottomrule
\end{tabular}
\end{table}

\section{Student entropy values}\label{app:entropy}

\Cref{tab:entropy} lists the student entropies plotted in \cref{fig:entropy}, with each method's ratio to OPSD.

\begin{table}[H]\centering\small\setlength{\tabcolsep}{3.4pt}
\caption{Student entropy $H(p)$ (nats per token) of the distilled student models; the Python runs follow the protocol noted in \cref{sec:setup}.
Each \emph{ratio} is relative to OPSD: the method's value divided by the OPSD value for the same model and task.
$<$ denotes a ratio below $0.1$.
\METHOD{} generally yields higher student entropy than OPSD.}
\label{tab:entropy}
\begin{adjustbox}{max width=\textwidth}\begin{tabular}{l rr rr rr rr}
\toprule
& \multicolumn{2}{c}{\textbf{Grading}} & \multicolumn{2}{c}{\textbf{Chemistry}} & \multicolumn{2}{c}{\textbf{SQL}} & \multicolumn{2}{c}{\textbf{Python}}\\
\cmidrule(lr){2-3}\cmidrule(lr){4-5}\cmidrule(lr){6-7}\cmidrule(lr){8-9}
Method & $H(p)$ & ratio & $H(p)$ & ratio & $H(p)$ & ratio & $H(p)$ & ratio\\
\midrule
\multicolumn{9}{l}{\textit{Fine-tuned model: Qwen3-4B}}\\[1pt]
OPSD & 0.10 & 1.0 & 0.07 & 1.0 & 0.03 & 1.0 & 0.32 & 1.0\\
AntiSD & 0.13 & 1.3 & 0.00 & $<$ & 0.10 & 3.3 & 0.40 & 1.2\\
\rowcolor{ptopsdblue}
\METHOD{} (single) & 0.09 & 0.9 & 0.49 & 7.0 & 0.49 & 16.0 & 0.49 & 1.5\\
\rowcolor{ptopsdblue}
\METHOD{} (multi) & 0.10 & 1.0 & 0.48 & 6.9 & 0.54 & 18.0 & 0.24 & 0.8\\
\midrule
\multicolumn{9}{l}{\textit{Fine-tuned model: OLMo-3-7B-Instruct}}\\[1pt]
OPSD & 0.89 & 1.0 & 0.81 & 1.0 & 0.19 & 1.0 & 0.58 & 1.0\\
AntiSD & 0.84 & 0.9 & 0.73 & 0.9 & 0.40 & 2.1 & 0.57 & 1.0\\
\rowcolor{ptopsdblue}
\METHOD{} (single) & 1.16 & 1.3 & 1.17 & 1.4 & 0.39 & 2.1 & 0.56 & 1.0\\
\rowcolor{ptopsdblue}
\METHOD{} (multi) & 1.43 & 1.6 & 1.36 & 1.7 & 0.66 & 3.5 & 0.60 & 1.0\\
\midrule
\multicolumn{9}{l}{\textit{Fine-tuned model: Qwen3-8B}}\\[1pt]
OPSD & 0.31 & 1.0 & 0.42 & 1.0 & 0.41 & 1.0 & 0.47 & 1.0\\
AntiSD & 0.06 & 0.2 & 0.12 & 0.3 & 0.03 & $<$ & 0.47 & 1.0\\
\rowcolor{ptopsdblue}
\METHOD{} (single) & 0.31 & 1.0 & 0.64 & 1.5 & 0.63 & 1.5 & 0.50 & 1.1\\
\rowcolor{ptopsdblue}
\METHOD{} (multi) & 0.78 & 2.5 & 1.19 & 2.8 & 0.53 & 1.3 & 0.48 & 1.0\\
\bottomrule
\end{tabular}\end{adjustbox}
\end{table}

\section{Teacher entropy values}\label{app:teacher-entropy}

\Cref{tab:teacher-entropy} reports the entropy of the teacher itself, measured under the exact conditioning used in training and without any training. For 120 test problems of each task we sample 8 student rollouts from the untrained Qwen3-4B at the training temperature, take a correct rollout of the same problem as the solution $c$ (as in training; 90 / 101 / 96 / 84 problems have one on grading / chemistry / SQL / Python), rebuild the teacher input for each condition (question and $c$; additionally the single prompt $d$ or one aspect prompt $d_k$ after $c$), and sample one teacher response per problem. $H(q)$ is the mean per-token Shannon entropy of the teacher over its response, estimated from the top-20 log-probabilities at each position (nats); the multi-prompt row is the mean over the aspect prompts, which are spread evenly over the rollouts in training. The prompts raise the teacher's entropy on every task: by 1.6 to 2.0 times with the single prompt and by 3.3 to
3.4 times with the aspect prompts on grading and chemistry

\begin{table}[H]\centering\small\setlength{\tabcolsep}{3.4pt}
\caption{Teacher entropy $H(q)$ (nats per token) of the untrained Qwen3-4B teacher, conditioned on the question and
a correct solution $c$ (OPSD) and additionally on the training prompt (\METHOD{}). Each \emph{ratio} is relative to the
OPSD teacher for the same task. $>$ denotes a ratio above 100: on SQL the OPSD teacher reproduces the query in
$c$ almost verbatim ($H(q){=}0.0001$). The prompts raise the teacher's entropy on every task; the two Python
prompts that ask for a minimal program ($d_1$, $d_2$) are the exception.}
\label{tab:teacher-entropy}
\begin{adjustbox}{max width=\textwidth}\begin{tabular}{l rr rr rr rr}
\toprule
& \multicolumn{2}{c}{\textbf{Grading}} & \multicolumn{2}{c}{\textbf{Chemistry}} & \multicolumn{2}{c}{\textbf{SQL}} & \multicolumn{2}{c}{\textbf{Python}}\\
\cmidrule(lr){2-3}\cmidrule(lr){4-5}\cmidrule(lr){6-7}\cmidrule(lr){8-9}
Teacher & $H(q)$ & ratio & $H(q)$ & ratio & $H(q)$ & ratio & $H(q)$ & ratio\\
\midrule
\multicolumn{9}{l}{\textit{Model: Qwen3-4B}}\\[1pt]
OPSD ($c$) & 0.024 & 1.0 & 0.038 & 1.0 & 0.000 & 1.0 & 0.077 & 1.0\\
\rowcolor{ptopsdblue}
\METHOD{} (single), $c+d$ & 0.041 & 1.7 & 0.075 & 2.0 & 0.053 & $>$ & 0.125 & 1.6\\
\rowcolor{ptopsdblue}
\METHOD{} (multi), mean over $c+d_k$ & 0.078 & 3.3 & 0.128 & 3.4 & 0.112 & $>$ & 0.093 & 1.2\\
\quad $c+d_1$ & 0.074 & 3.1 & 0.131 & 3.4 & 0.079 & $>$ & 0.030 & 0.4\\
\quad $c+d_2$ & 0.101 & 4.2 & 0.131 & 3.4 & 0.127 & $>$ & 0.035 & 0.5\\
\quad $c+d_3$ & 0.052 & 2.2 & 0.124 & 3.3 & 0.035 & $>$ & 0.166 & 2.1\\
\quad $c+d_4$ & 0.086 & 3.6 & \multicolumn{2}{c}{---} & 0.206 & $>$ & 0.142 & 1.8\\
\bottomrule
\end{tabular}\end{adjustbox}
\end{table}

\section{Prompt-set ablation: all metrics}\label{app:promptset}
\Cref{tab:promptset} lists Pass@1, Maj@16 and Worst@16 for every cell of the ablation of \cref{sec:promptset}; \cref{fig:promptset} plots the Pass@1 column.

\begin{table}[ht]\centering\small\setlength{\tabcolsep}{4.5pt}
\caption{Prompt sets for the teacher on Qwen3-4B (\%), full numbers for \cref{fig:promptset}. $d_k$ alone: every rollout is trained toward the teacher prompted with that one aspect prompt. Concatenation: $d_1,\ldots,d_K$ joined into a single prompt. Chemistry only has $K{=}3$.}
\label{tab:promptset}
\begin{adjustbox}{max width=\textwidth}\begin{tabular}{l rrr rrr rrr rrr}
\toprule
& \multicolumn{3}{c}{\textbf{Grading}} & \multicolumn{3}{c}{\textbf{Chemistry}} & \multicolumn{3}{c}{\textbf{SQL}} & \multicolumn{3}{c}{\textbf{Python}}\\
\cmidrule(lr){2-4}\cmidrule(lr){5-7}\cmidrule(lr){8-10}\cmidrule(lr){11-13}
Teacher prompt & P@1 & M@16 & W@16 & P@1 & M@16 & W@16 & P@1 & M@16 & W@16 & P@1 & M@16 & W@16\\
\midrule
No prompt (OPSD) & 71.9 & 71.8 & 61.3 & 81.7 & 81.9 & 79.2 & 69.8 & 69.7 & 64.9 & 61.4 & 61.9 & 51.8\\
$d_1$ alone & 71.5 & 71.8 & 63.3 & 84.1 & 84.3 & 79.7 & 79.1 & 79.8 & 72.4 & 60.5 & 61.3 & 51.1\\
$d_2$ alone & 74.0 & 73.8 & 65.2 & 81.5 & 81.5 & 78.8 & 80.5 & 81.2 & 73.5 & 62.3 & 63.5 & 53.2\\
$d_3$ alone & 78.3 & 78.8 & 73.1 & 73.9 & 76.3 & 53.0 & 72.9 & 73.0 & 68.1 & 63.3 & 64.5 & \B{55.2}\\
$d_4$ alone & 72.7 & 76.0 & 43.7 & --- & --- & --- & 76.0 & 77.0 & 64.1 & 62.9 & 65.1 & 48.1\\
$d_1,\ldots,d_K$ concatenated & 67.4 & 67.4 & 65.6 & 80.4 & 80.9 & 73.7 & 79.2 & 80.4 & 70.9 & 60.6 & 60.5 & 52.1\\
\midrule
\rowcolor{ptopsdblue}
Single prompt $d$ & 78.7 & 78.7 & 73.8 & \B{86.1} & \B{86.3} & \B{81.8} & 80.0 & 80.6 & 71.7 & \B{64.1} & \B{65.2} & 54.5\\
\rowcolor{ptopsdblue}
Prompt set $\{d_k\}$ & \B{83.0} & \B{83.5} & \B{76.6} & 85.8 & 86.2 & 81.3 & \B{81.0} & \B{81.7} & \B{74.4} & 63.8 & 64.7 & 54.8\\
\bottomrule
\end{tabular}\end{adjustbox}
\end{table}

\section{Control: the teacher prompt without privileged information}\label{app:noc}

\METHOD{} adds a system prompt $d$ to a teacher that already holds the privileged information $c$. To test whether
$d$ could itself act as privileged information, we train the multi-prompt setting on Qwen3-4B with a teacher that
sees the question and the aspect prompts $d_k$ but \emph{no} $c$, under the protocol of \cref{sec:setup} and with
the same rollouts entering the distillation loss. The teacher then knows nothing that the student does not, apart
from the instructions in $d_k$. \Cref{fig:noc} compares this control with the untrained model, OPSD and \METHOD{}.

\begin{figure}[H]\centering
\includegraphics[width=\textwidth]{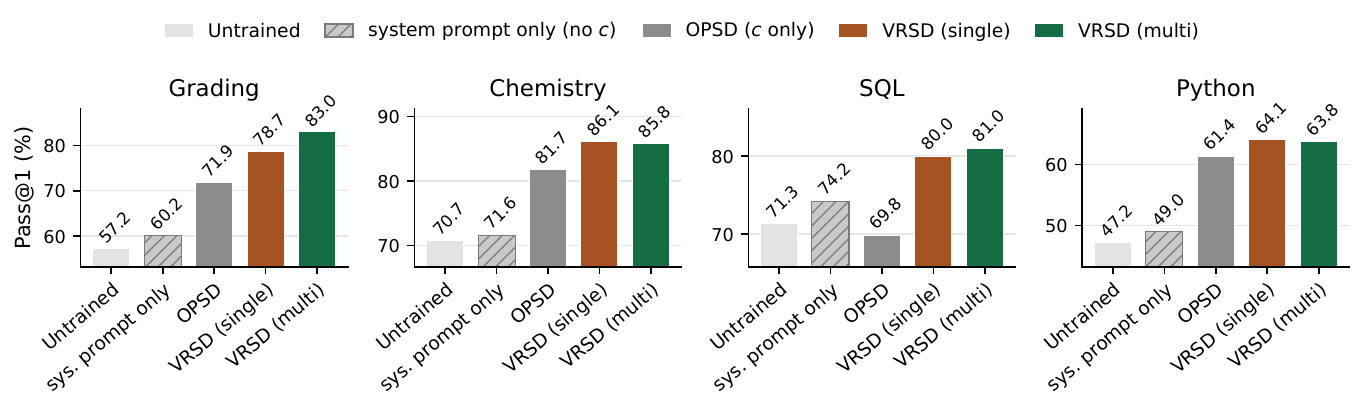}
\vspace{-4mm}
\caption{Pass@1 (\%) on Qwen3-4B of the control whose teacher gets only the system prompts (aspect prompts, no
privileged information $c$), next to the untrained model, OPSD ($c$ only) and \METHOD{} ($c$ and system prompt).}
\label{fig:noc}
\end{figure}

\textbf{A teacher-only system prompt is not privileged information.} Without $c$, the prompted teacher improves
the student by only 1 to 3 points over the untrained model (grading 60.2, chemistry 71.6, SQL 74.2, Python 49.0),
10 to 12 points below OPSD on three of the four tasks and 7 to 23 points below \METHOD{} (multi), although its
prompts carry the full task procedure (rubric checklists, schema and join planning, boundary checks). The
supervisory signal therefore comes from $c$, and the prompt does not replace it.

\textbf{The system prompt acts as a regularizer of how the teacher uses $c$.} Given $c$, the same prompts raise
the student above OPSD on every task, keep the teacher as accurate as the OPSD teacher (\cref{app:teacher-entropy})
and raise its average entropy on every task. The prompt thus does not add information about the answer; it shapes
the distribution through which the answer in $c$ is distilled, keeping the teacher from collapsing onto the
solution and passing the verification and reconsideration steps to the student
(\cref{sec:qualitative}). This is the role assumed by our analysis, in which $d$ attenuates the influence of the
privileged information on the teacher rather than supplying information of its own.

\section{Training prompts}\label{app:prompts}

The prompts $d$ used in training, one block per task: the single prompt $d$ followed by the
prompts $d_1,\ldots,d_K$ of the multi-prompt setting. The same prompts are used for all three models,
except that OLMo-3-7B-Instruct and Qwen3-8B use the second version of $d_4$ on grading and of $d_3$ on chemistry.

\medskip\noindent\textbf{Grading (ASAP-SAS), $K{=}4$.}\par\nopagebreak
\begin{lstlisting}[basicstyle=\ttfamily\scriptsize, breaklines=true, breakatwhitespace=true, breakindent=0pt, columns=fullflexible, keepspaces=true, frame=single, framerule=0.3pt, xleftmargin=2pt, xrightmargin=2pt, aboveskip=2pt, belowskip=2pt, escapeinside={(*@}{@*)}, literate={—}{{\textrm{---}}}1]
(*@\textbf{$d$}@*)  You are a careful, calibrated examiner. Before committing to a mark, work through a disciplined grading routine and resist the urge to guess the most common score. Follow these steps in your reasoning:
1. Restate each distinct credit-worthy element the rubric requires, as an explicit checklist.
2. Go through the student's answer and, for every rubric element, quote or point to the exact span that satisfies it — or note that nothing does. Do not credit an element on a vague impression; require concrete evidence in the student's words.
3. Count only the elements you found genuine evidence for. Watch for two failure modes: giving credit the answer did not earn (over-crediting), and missing a valid point phrased in the student's own imperfect wording (under-crediting).
4. Map that count to the rubric's mark bands, then sanity-check: would a second examiner reading the same spans agree? Reconsider once if unsure.
Only after this routine, output the final integer mark.

(*@\textbf{$d_1$}@*)  You are an examiner who grades strictly by the rubric's structure. Before giving a mark, first rewrite the rubric as an explicit numbered checklist of the distinct credit-worthy elements it requires. Then walk the checklist top to bottom and decide, element by element, whether the student's answer satisfies it. Do not think about the overall impression or the final number until every element has been judged individually. Only then count the satisfied elements and map that count to the rubric's mark bands.

(*@\textbf{$d_2$}@*)  You are an examiner guarding against anchoring on the most common score. Before giving a mark, explicitly consider the full range of marks the rubric allows, including the extremes. State what an answer at the top band and an answer at the bottom band would look like for this question, then ask honestly which one this student's answer is closer to. Resist the pull toward the middle or toward whichever score feels most typical; let the answer's actual content decide the band.

(*@\textbf{$d_3$}@*)  You are an examiner who credits nothing without quoted evidence. Before giving a mark, for every point you intend to award, quote the exact span of the student's answer that earns it. If you cannot quote a span, you may not award that point. Be careful to recognise valid content expressed in the student's own imperfect wording, spelling, or phrasing — quote it anyway and credit it. Award the mark strictly from the quoted spans you collected.

(*@\textbf{$d_4$}@*)  You are an examiner who challenges your own first judgement. Form a provisional mark, then argue against it: build the strongest case that the mark is too high (what did you credit that the answer never actually earned?) and the strongest case that it is too low (what valid point did you overlook or dismiss for poor wording?). Weigh both counter-cases explicitly before settling. Change the provisional mark whenever the counter-argument is the stronger one.
\end{lstlisting}

\noindent The second version of $d_4$ (OLMo-3-7B-Instruct, Qwen3-8B):
\begin{lstlisting}[basicstyle=\ttfamily\scriptsize, breaklines=true, breakatwhitespace=true, breakindent=0pt, columns=fullflexible, keepspaces=true, frame=single, framerule=0.3pt, xleftmargin=2pt, xrightmargin=2pt, aboveskip=2pt, belowskip=2pt, escapeinside={(*@}{@*)}, literate={—}{{\textrm{---}}}1]
(*@\textbf{$d_4 \text{ (v2)}$}@*)  You are an examiner who tests a mark against its nearest alternative. Once you have a mark in mind, name the adjacent mark a second examiner would most plausibly argue for, and state exactly what would have to be true of the student's answer for THAT mark to be right — which further rubric element would have to be credited, or which credited element would have to be withdrawn. Then check that condition against what the student actually wrote and show precisely why it does not hold. Adjust your mark only if that check uncovers a rubric element you genuinely mis-scored, not merely because the alternative can be argued for.
\end{lstlisting}

\medskip\noindent\textbf{Text-to-SQL (Spider), $K{=}4$.}\par\nopagebreak
\begin{lstlisting}[basicstyle=\ttfamily\scriptsize, breaklines=true, breakatwhitespace=true, breakindent=0pt, columns=fullflexible, keepspaces=true, frame=single, framerule=0.3pt, xleftmargin=2pt, xrightmargin=2pt, aboveskip=2pt, belowskip=2pt, escapeinside={(*@}{@*)}, literate={—}{{\textrm{---}}}1]
(*@\textbf{$d$}@*)  You are an expert SQL analyst. Before committing to a query, work through a disciplined routine and resist writing SQL from the question's surface wording. Follow these steps in your reasoning:
1. Ground every entity in the question to the actual schema: name the exact tables and columns you will use, and note their types. Do not invent a column that is not in the schema; map synonyms in the question to the real column names.
2. Work out the join path explicitly: for every pair of tables you touch, state the foreign-key columns that link them, and join only along those keys so no rows are lost or multiplied.
3. Pin down the filtering and shaping: translate each condition in the question into the correct WHERE predicate (watch value matching, text vs number, NULLs), and decide GROUP BY / aggregate / HAVING / ORDER BY / LIMIT deliberately — a "how many" needs COUNT, a "most/least/top" needs ORDER BY with LIMIT, a "per X" needs GROUP BY X.
4. Draft the query, then verify it by tracing what it returns: does the result answer exactly what was asked, at the right granularity, with no missing or duplicated rows? Check the boundary cases (empty result, ties, duplicates). Revise the draft whenever the trace disagrees with the question.
Only after this routine, output the final query inside a ```sql code block.

(*@\textbf{$d_1$}@*)  You are an analyst who grounds everything in the schema before writing any SQL. Before drafting the query, list the exact tables and columns you will use and note their types, reading them straight from the schema. Map every entity in the question to a real column name — never invent a column that is not there, and translate the question's synonyms into the schema's actual names. Only once every needed field is pinned to a real schema column do you write the query.

(*@\textbf{$d_2$}@*)  You are an analyst who gets the join path right above all else. Before drafting the query, identify every table the answer must touch, and for each pair of them state explicitly the foreign-key columns that link them. Join only along those keys, and check that no join loses rows that should be kept or multiplies rows that should be unique. Prefer the shortest correct join path through the foreign keys; do not connect tables on columns that are not genuine keys.

(*@\textbf{$d_3$}@*)  You are an analyst who is meticulous about filtering and shaping the result. Before drafting the query, translate each condition in the question into the correct WHERE predicate, watching for value matching, text-vs-number comparisons, and NULLs. Then decide the aggregation and shaping deliberately: a "how many" needs COUNT, a "most / least / top-k" needs ORDER BY with LIMIT, a "per X" or "for each X" needs GROUP BY X, and a condition on an aggregate needs HAVING. Make the granularity of the output match exactly what the question asks for.

(*@\textbf{$d_4$}@*)  You are an analyst who never trusts a first draft. Write a provisional query, then verify it by tracing what it actually returns: does the result answer exactly what was asked, at the right granularity, with no missing or duplicated rows? Walk the boundary cases — an empty result, ties at the top, duplicate rows, NULLs — and ask whether the query still behaves correctly. Whenever the trace disagrees with the question, revise the query and trace again before settling.
\end{lstlisting}

\medskip\noindent\textbf{LiveCodeBench, $K{=}4$.}\par\nopagebreak
\begin{lstlisting}[basicstyle=\ttfamily\scriptsize, breaklines=true, breakatwhitespace=true, breakindent=0pt, columns=fullflexible, keepspaces=true, frame=single, framerule=0.3pt, xleftmargin=2pt, xrightmargin=2pt, aboveskip=2pt, belowskip=2pt, escapeinside={(*@}{@*)}, literate={—}{{\textrm{---}}}1]
(*@\textbf{$d$}@*)  You are a careful, reflective problem solver. Before committing to an answer, pause and reflect on your reasoning: re-examine each step for mistakes, question assumptions, check your arithmetic, and consider whether your earlier attempt went wrong and why. Only then give your final answer.

(*@\textbf{$d_1$}@*)  You are a solver who fixes the input/output contract before writing a single line of logic. From the sample cases, determine exactly how input arrives (stdin lines? a function signature?) and exactly what must be printed (format, ordering, trailing newline). If the samples show plain stdin/stdout, write a program that reads stdin and prints — never a bare `class Solution` method that only returns a value, or the harness sees no output. Echo the contract back to yourself, then code to it. A solution that violates the I/O contract fails every test no matter how correct the algorithm.

(*@\textbf{$d_2$}@*)  You are a solver who adds no transformation the problem did not ask for. Model exactly the quantities the statement defines — no unit conversions (minutes to seconds, 1-based to 0-based) and no scaling unless a line of the problem explicitly demands it. Every operation in your code must trace back to a specific sentence in the specification; if you cannot point to that sentence, delete the operation. Most regressions come from a plausible-looking extra step that the base solution never needed.

(*@\textbf{$d_3$}@*)  You are a solver who treats loop bounds and indices as the likeliest bug. Before finalizing, re-derive the exact range of every loop and the exact index of every access: does it start at 0 or 1, is the last element included, does `range(n-1)` versus `range(1,n)` shift the pairing you intended? Check the first element, the last element, and the single-element input by hand. An off-by-one that passes the big samples still fails the boundary tests, so verify the edges explicitly.

(*@\textbf{$d_4$}@*)  You are a solver who hand-executes every provided sample before committing. For each sample input, step through your actual code line by line — not your intention — and produce the output it would really print, then compare to the expected output character by character. The cases where they differ are exactly where your model of the problem is wrong; fix the logic, do not patch the symptom, and re-trace. Do not finalize until your code reproduces every sample output exactly.
\end{lstlisting}

\section{Response-Pattern Analysis}\label{app:teachercont}\label{app:behavior-heuristics}

\subsection{Setup}\label{app:pattern-setup}

A response pattern is defined as a deterministic rule $r$ evaluated over a model response $x$, where $r(x) \in \{0, 1\}$. For a set of responses $X$, we quantify its prevalence via the empirical frequency:
\begin{equation}
F_r(X) = \frac{100}{|X|} \sum_{x \in X} r(x) \quad (\%),
\end{equation}
evaluated entirely without model-based judges to ensure zero-cost reproducibility. For rules that inspect the reasoning process, matching is restricted strictly to the reasoning portion of $x$: text inside the \texttt{<reasoning>} tags for grading and chemistry, and text preceding the final code block for SQL and Python. All text matching is case-insensitive. \Cref{tab:pattern-rules} summarizes the exact pattern rules; length denotes the word count of the reasoning segment.

We apply these heuristics across two distinct response sets:
(i)~\emph{Student answers}: evaluating all test-set rollouts (16 samples per problem) from the evaluated checkpoint of each student model (\cref{tab:patterns}).
(ii)~\emph{Teacher continuations}: evaluating how the teacher completes partially generated student rollouts. Specifically, a student answer is cut in half at a sentence boundary, and the untrained Qwen3-4B is tasked with continuing it under three conditioning contexts: (1)~the question alone, (2)~the question with a correct solution $c$ (the OPSD teacher), or (3)~the question and $c$ augmented with either the single prompt $d$ or an aspect prompt $d_k$ (\METHOD{} teachers). Here, $c$ is a verified rollout of the same problem, identical to training. One continuation per problem is sampled at the training temperature across 101 chemistry, 93 grading, 96 SQL, and 85 Python problems, and scored with task-specific verifiers (\cref{tab:teachercont}).

\begin{table}[ht]\centering\small\setlength{\tabcolsep}{4pt}
\caption{Pattern rules. A rule fires when the stated condition holds.}
\label{tab:pattern-rules}
\begin{tabular}{l l p{8.1cm}}
\toprule
Task & Pattern & Rule \\
\midrule
all & Contrast & A contrastive turn: \emph{but, however, although, though, yet, instead, rather, whereas, nonetheless, actually, wait} \\
\midrule
Grading & Checks rubric elements & At least two enumerated items (\texttt{1.}, \texttt{2)}, \emph{element/point/criterion $n$}) in the reasoning \\
 & Quotes student evidence & A quoted span of $\ge 8$ characters that occurs verbatim in the prompt's student answer (case and punctuation ignored) \\
 & Rejects other scores & \emph{does not meet/earn/qualify/warrant, would earn/receive, cannot be awarded, not enough/sufficient for, not worth} \\
\midrule
Chemistry & Discusses all four options & Each of options A--D is referred to individually (\emph{A)}, \emph{option B}, \emph{C is}, \emph{D:}) \\
 & Rules out options & \emph{rule(s) out, eliminate, incorrect because, is wrong because, does not fit/match/apply} \\
 & Checks conclusion & \emph{verify, double-check, re-examine, re-evaluate, re-check, confirm} (any inflection) \\
\midrule
SQL & Names schema tables & At least two table names declared by the prompt's \texttt{CREATE TABLE} statements \\
 & Names join keys & \emph{foreign key}, or \emph{join \ldots on/via/through/using}, or \emph{link \ldots table/key/column} \\
 & Plans query clauses & \emph{group by, having, order by, limit $n$, aggregate} \\
\midrule
Python & Code without preamble & The response opens directly with a markdown code block \\
 & Discusses boundary cases & \emph{edge/corner case, boundary, off-by-one, empty input/list/array/string, single element, $n=0$, $n=1$} \\
 & Checks conclusion & As for chemistry \\
\midrule
SQL, Python & Verifies & Tests the partial draft: \emph{verify, check, double-check, edge case, boundary, sample input/case, test it/this/the, trace, walk through} \\
\bottomrule
\end{tabular}
\end{table}

\subsection{Results: What the Prompt Changes in the Teacher's Text}\label{app:teacher_reason}

\begin{table}[ht]\centering\small\setlength{\tabcolsep}{5pt}
\caption{Teacher continuations of half-written student answers by the untrained Qwen3-4B (\%). Columns indicate the context seen by the teacher before the prefix. Correct: the completed answer is correct or passes unit tests. Other rows denote pattern frequency $F_r$. Chemistry evaluates three aspect prompts.}
\label{tab:teachercont}
\begin{adjustbox}{max width=\textwidth}\begin{tabular}{ll r r >{\columncolor{ptopsdblue}}r >{\columncolor{ptopsdblue}}r >{\columncolor{ptopsdblue}}r >{\columncolor{ptopsdblue}}r >{\columncolor{ptopsdblue}}r}
\toprule
& & \multicolumn{1}{c}{question} & \multicolumn{1}{c}{OPSD teacher} & \multicolumn{5}{c}{\METHOD{} teacher} \\
\cmidrule(lr){3-3}\cmidrule(lr){4-4}\cmidrule(lr){5-9}
Task & & only & $+\,c$ & $+\,c+d$ & $+\,c+d_1$ & $+\,c+d_2$ & $+\,c+d_3$ & $+\,c+d_4$ \\
\midrule
\multirow{3}{*}{Chemistry} & Correct & 79 & 95 & 94 & 95 & 94 & 74 & -- \\
 & Contrast & 50 & 45 & 48 & 67 & 48 & 100 & -- \\
 & Rules out options & 3 & 2 & 6 & 39 & 8 & 6 & -- \\
\midrule
\multirow{3}{*}{Grading} & Correct & 73 & 96 & 96 & 91 & 91 & 89 & 76 \\
 & Contrast & 71 & 66 & 67 & 60 & 66 & 61 & 94 \\
 & Rejects other scores & 30 & 30 & 33 & 23 & 22 & 27 & 41 \\
\midrule
\multirow{3}{*}{SQL} & Correct & 87 & 99 & 100 & 100 & 98 & 99 & 96 \\
 & Contrast & 0 & 2 & 4 & 7 & 6 & 6 & 43 \\
 & Verifies & 5 & 6 & 16 & 4 & 6 & 9 & 73 \\
\midrule
\multirow{3}{*}{Python} & Correct & 62 & 74 & 76 & 78 & 76 & 74 & 77 \\
 & Contrast & 41 & 32 & 40 & 32 & 33 & 46 & 41 \\
 & Verifies & 49 & 52 & 55 & 53 & 48 & 53 & 78 \\
\bottomrule
\end{tabular}\end{adjustbox}
\end{table}

\textbf{The solution supplies the answer, not the reasoning.} 
Conditioning on privileged information $c$ substantially increases continuation accuracy across chemistry, grading, and SQL, but leaves qualitative reasoning behaviors largely unchanged (\cref{tab:teachercont}, first two columns). Without prompting, the teacher produces continuations of comparable length to the unconditioned baseline while exhibiting flat or lower rates of contrastive turns, option elimination, and draft verification. In SQL, this shortcut dependence is explicit: given an unconstrained prefix, the OPSD teacher emits the privileged query verbatim. On Python, where solutions cannot be trivially pasted into existing code, the teacher actively sheds intermediate rationales and jumps straight to implementation. This directly illustrates the shortcut bias formalized in \cref{sec:setup}: privileged information accelerates the teacher onto the final conclusion while bypassing the intermediate reasoning steps that justify it.

\textbf{Prompts restore structured reasoning steps.} 
System prompts systematically restore these intermediate reasoning steps while preserving high continuation accuracy, shaping the continuation in direct alignment with the prompt's semantic objective. On chemistry, the elimination prompt ($d_1$) drives the teacher to explicitly reject incorrect options in over a third of rollouts (compared to negligible rates under OPSD) without compromising correctness. On SQL and Python, verification prompts ($d_4$) induce active sanity checking and test tracing across the majority of partial drafts. These teacher-level behavioral shifts provide the direct supervisory signal that transfers to the distilled student policies in \cref{tab:patterns}. 

The notable exceptions are adversarial prompts (chemistry $d_3$ and grading $d_4$), which explicitly instruct the teacher to challenge its provisional pick. These prompts induce a severe drop in continuation accuracy despite having the correct solution in context, directly explaining the teacher-side origin of the student Worst@16 collapse documented in \cref{tab:promptset}. On Python, no prompt degrades accuracy below the OPSD baseline. \Cref{fig:teachercont} illustrates this contrast with a full rollout.

\begin{figure}[ht]
\begin{tcolorbox}[teacherpromptbox, colframe=black!40, coltitle=black,
    title={Chemistry, gold C. Student prefix, followed by continuations under $c$ and $c+d_1$.}]
\begin{lstlisting}[style=teacherprompttext]
(*@\textbf{Prefix:}@*)  Hydrophobic interactions at chemically heterogeneous interfaces are characterized by the tendency of nonpolar molecules to avoid aqueous environments. The unique signature of hydrophobicity in such contexts is often related to the behavior of water molecules around nonpolar surfaces.

(*@\textbf{OPSD teacher ($+c$):}@*)  Water molecules tend to form highly coordinated, ordered structures around hydrophobic regions to minimize contact with nonpolar surfaces, which is a key feature of hydrophobic interactions. 
Option A refers to polar end groups, which are not a unique signature of hydrophobicity. Option B refers to the chemical identity of the interfacial molecules, which is not specific to hydrophobicity. Option D refers to complex, non-additive contributions, which is a general description rather than a unique signature. Therefore, the correct answer is C.

(*@\textbf{Prompted teacher ($+c+d_1$):}@*)  Option A: Presence of polar end groups. This is incorrect because polar end groups would actually enhance hydrophilicity, not hydrophobicity. Hydrophobicity is associated with the absence of polar groups, not their presence.

Option B: Chemical identity of the interfacial molecules. While chemical identity does influence hydrophobic interactions, it is not a unique signature of hydrophobicity itself. The chemical identity is a factor, but not the defining characteristic of hydrophobicity in heterogeneous interfaces.

Option C: Highly coordinated water structures. This is correct. Water molecules tend to form highly ordered, structured arrangements around nonpolar molecules to minimize contact with nonpolar surfaces. This ordered water structure is a unique signature of hydrophobicity in heterogeneous interfaces.

Option D: Complex, non-additive contributions. While hydrophobic interactions can involve complex contributions, this is more of a general description rather than a unique signature. The specific mechanism of highly coordinated water structures is more directly tied to the phenomenon of hydrophobicity.

After considering all options, the unique signature of hydrophobicity at chemically heterogeneous interfaces is the highly coordinated water structures. This is a specific and defining characteristic of hydrophobic interactions, making option C the correct answer.
\end{lstlisting}
\end{tcolorbox}
\caption{Illustrative teacher continuations on a half-written chemistry problem (gold answer C) completed by untrained Qwen3-4B with solution $c$ (OPSD teacher) and with $c$ plus the elimination prompt $d_1$ (verbatim, reasoning tags removed). While both reach the correct answer, the prompted teacher explicitly evaluates and rules out alternative options.}
\label{fig:teachercont}
\end{figure}

%% file: main.bbl
\begin{thebibliography}{42}
\providecommand{\natexlab}[1]{#1}
\providecommand{\url}[1]{\texttt{#1}}
\expandafter\ifx\csname urlstyle\endcsname\relax
  \providecommand{\doi}[1]{doi: #1}\else
  \providecommand{\doi}{doi: \begingroup \urlstyle{rm}\Url}\fi

\bibitem[Agarwal et~al.(2024)Agarwal, Vieillard, Zhou, Stanczyk, Garea, Geist,
  and Bachem]{opd1}
Rishabh Agarwal, Nino Vieillard, Yongchao Zhou, Piotr Stanczyk, Sabela~Ramos
  Garea, Matthieu Geist, and Olivier Bachem.
\newblock On-policy distillation of language models: Learning from
  self-generated mistakes.
\newblock In \emph{The Twelfth International Conference on Learning
  Representations}, 2024.
\newblock URL \url{https://openreview.net/forum?id=3zKtaqxLhW}.

\bibitem[Barbara et~al.(2012)Barbara, Hamner, Morgan, lynnvandev, and
  Shermis]{hamner2012asapsas}
Barbara, Ben Hamner, Jaison Morgan, lynnvandev, and Mark Shermis.
\newblock The hewlett foundation: Short answer scoring.
\newblock \url{https://kaggle.com/competitions/asap-sas}, 2012.
\newblock Kaggle.

\bibitem[{Core Team} et~al.(2026){Core Team}, Xiao, Xia, Yang, Gao, Shen,
  Zhang, He, Lou, Luo, Wang, Xie, Zhang, Lv, Li, Chen, Xu, Zhang, Liu, Duo,
  et~al.]{mopd1}
{Core Team}, Bangjun Xiao, Bingquan Xia, Bo~Yang, Bofei Gao, Bowen Shen, Chen
  Zhang, Chenhong He, Chiheng Lou, Fuli Luo, Gang Wang, Gang Xie, Hailin Zhang,
  Hanglong Lv, Hanyu Li, Heyu Chen, Hongshen Xu, Houbin Zhang, Huaqiu Liu,
  Jiangshan Duo, et~al.
\newblock Mimo-v2-flash technical report, 2026.
\newblock URL \url{https://arxiv.org/abs/2601.02780}.

\bibitem[{DeepSeek-AI} et~al.(2026){DeepSeek-AI}, Xu, Lin, Xue, Wang, Xu, Wu,
  Zhang, Lin, Dong, Ling, Lu, Zhao, Deng, Hou, Xu, Shao, Ruan, Sun, Dai,
  et~al.]{mopd2}
{DeepSeek-AI}, Anyi Xu, Bangcai Lin, Bing Xue, Bingxuan Wang, Bingzheng Xu,
  Bochao Wu, Bowei Zhang, Chaofan Lin, Chen Dong, Chenchen Ling, Chengda Lu,
  Chenggang Zhao, Chengqi Deng, Chengyu Hou, Chenhao Xu, Chenze Shao, Chong
  Ruan, Conner Sun, Damai Dai, et~al.
\newblock Deepseek-v4: Towards highly efficient million-token context
  intelligence, 2026.
\newblock URL \url{https://arxiv.org/abs/2606.19348}.

\bibitem[Feng et~al.(2025)Feng, Shen, Wang, Zhuang, Tang, Zhang, and
  Ding]{feng2024sciknoweval}
Kehua Feng, Xinyi Shen, Weijie Wang, Xiang Zhuang, Yuqi Tang, Qiang Zhang, and
  Keyan Ding.
\newblock Sciknoweval: A comprehensive dataset for evaluating scientific
  knowledge of large language models.
\newblock In \emph{NeurIPS 2025 AI for Science Workshop}, 2025.
\newblock URL \url{https://openreview.net/forum?id=evsup2BiKo}.

\bibitem[Gou et~al.(2021)Gou, Yu, Maybank, and Tao]{Gou_2021}
Jianping Gou, Baosheng Yu, Stephen~J. Maybank, and Dacheng Tao.
\newblock Knowledge distillation: A survey.
\newblock \emph{International Journal of Computer Vision}, 129\penalty0
  (6):\penalty0 1789--1819, March 2021.
\newblock ISSN 1573-1405.
\newblock \doi{10.1007/s11263-021-01453-z}.
\newblock URL \url{http://dx.doi.org/10.1007/s11263-021-01453-z}.

\bibitem[Gu et~al.(2024)Gu, Dong, Wei, and Huang]{opd2}
Yuxian Gu, Li~Dong, Furu Wei, and Minlie Huang.
\newblock Mini{LLM}: Knowledge distillation of large language models.
\newblock In \emph{The Twelfth International Conference on Learning
  Representations}, 2024.
\newblock URL \url{https://openreview.net/forum?id=5h0qf7IBZZ}.

\bibitem[Han et~al.(2026)Han, Zhang, Wang, and Sun]{atesd}
Zihao Han, Tiangang Zhang, Huaibin Wang, and Yilun Sun.
\newblock Adaptive teacher exposure for self-distillation in llm reasoning,
  2026.
\newblock URL \url{https://arxiv.org/abs/2605.11458}.

\bibitem[Harne et~al.(2026)Harne, Karkar, Pandya, Awadallah, and
  Nambi]{harne2026privilegedbiasedpiconditionedteachers}
Sarthak Harne, Chinmay Karkar, Yash Pandya, Ahmed Awadallah, and Akshay Nambi.
\newblock Privileged, but biased: How pi-conditioned teachers break
  self-distillation, 2026.
\newblock URL \url{https://arxiv.org/abs/2608.04794}.

\bibitem[Hinton et~al.(2015)Hinton, Vinyals, and
  Dean]{hinton2015distillingknowledgeneuralnetwork}
Geoffrey Hinton, Oriol Vinyals, and Jeff Dean.
\newblock Distilling the knowledge in a neural network, 2015.
\newblock URL \url{https://arxiv.org/abs/1503.02531}.

\bibitem[Ho et~al.(2023)Ho, Schmid, and Yun]{ho-etal-2023-large}
Namgyu Ho, Laura Schmid, and Se-Young Yun.
\newblock Large language models are reasoning teachers.
\newblock In Anna Rogers, Jordan Boyd-Graber, and Naoaki Okazaki (eds.),
  \emph{Proceedings of the 61st Annual Meeting of the Association for
  Computational Linguistics (Volume 1: Long Papers)}, pp.\  14852--14882,
  Toronto, Canada, July 2023. Association for Computational Linguistics.
\newblock \doi{10.18653/v1/2023.acl-long.830}.
\newblock URL \url{https://aclanthology.org/2023.acl-long.830/}.

\bibitem[Hsieh et~al.(2023)Hsieh, Li, Yeh, Nakhost, Fujii, Ratner, Krishna,
  Lee, and Pfister]{hsieh-etal-2023-distilling}
Cheng-Yu Hsieh, Chun-Liang Li, Chih-kuan Yeh, Hootan Nakhost, Yasuhisa Fujii,
  Alex Ratner, Ranjay Krishna, Chen-Yu Lee, and Tomas Pfister.
\newblock Distilling step-by-step! outperforming larger language models with
  less training data and smaller model sizes.
\newblock In Anna Rogers, Jordan Boyd-Graber, and Naoaki Okazaki (eds.),
  \emph{Findings of the Association for Computational Linguistics: ACL 2023},
  pp.\  8003--8017, Toronto, Canada, July 2023. Association for Computational
  Linguistics.
\newblock \doi{10.18653/v1/2023.findings-acl.507}.
\newblock URL \url{https://aclanthology.org/2023.findings-acl.507/}.

\bibitem[H{\"u}botter et~al.(2026)H{\"u}botter, L{\"u}beck, Behric, Baumann,
  Bagatella, Marta, Hakimi, Shenfeld, Buening, Guestrin, and Krause]{sdpo}
Jonas H{\"u}botter, Frederike L{\"u}beck, Lejs~Deen Behric, Anton Baumann,
  Marco Bagatella, Daniel Marta, Ido Hakimi, Idan Shenfeld, Thomas~Kleine
  Buening, Carlos Guestrin, and Andreas Krause.
\newblock Reinforcement learning via self-distillation.
\newblock In \emph{Forty-third International Conference on Machine Learning},
  2026.
\newblock URL \url{https://openreview.net/forum?id=QkfkxyRizZ}.

\bibitem[Jain et~al.(2025)Jain, Han, Gu, Li, Yan, Zhang, Wang, Solar-Lezama,
  Sen, and Stoica]{jain2025livecodebench}
Naman Jain, King Han, Alex Gu, Wen-Ding Li, Fanjia Yan, Tianjun Zhang, Sida
  Wang, Armando Solar-Lezama, Koushik Sen, and Ion Stoica.
\newblock Livecodebench: Holistic and contamination free evaluation of large
  language models for code.
\newblock In \emph{The Thirteenth International Conference on Learning
  Representations}, 2025.
\newblock URL \url{https://openreview.net/forum?id=chfJJYC3iL}.

\bibitem[Jiang et~al.(2026)Jiang, Xie, Jiang, and
  Zhang]{jiang2026distillskillsweightsprompts}
Yubo Jiang, Fengying Xie, Zhiguo Jiang, and Haopeng Zhang.
\newblock Distill skills into weights, not prompts: Abstract skills as
  privileged signals for on-policy self-distillation, 2026.
\newblock URL \url{https://arxiv.org/abs/2608.09826}.

\bibitem[Kaur et~al.(2026)Kaur, Ri, He, Fowl, and
  Arora]{kaur2026rethinkingonpolicyselfdistillationthinking}
Simran Kaur, Narutatsu Ri, Yinghui He, Liam Fowl, and Sanjeev Arora.
\newblock Rethinking on-policy self-distillation for thinking models, 2026.
\newblock URL \url{https://arxiv.org/abs/2607.05184}.

\bibitem[Kim et~al.(2026)Kim, Luo, Kim, Lee, Kim, Jeon, Li, and
  Yang]{kim2026doesselfdistillationsometimesdegrade}
Jeonghye Kim, Xufang Luo, Minbeom Kim, Sangmook Lee, Dohyung Kim, Jiwon Jeon,
  Dongsheng Li, and Yuqing Yang.
\newblock Why does self-distillation (sometimes) degrade the reasoning
  capability of {LLM}s?
\newblock In \emph{Third Conference on Language Modeling}, 2026.
\newblock URL \url{https://openreview.net/forum?id=Az7guis46K}.

\bibitem[Kim \& Rush(2016)Kim and Rush]{kim-rush-2016-sequence}
Yoon Kim and Alexander~M. Rush.
\newblock Sequence-level knowledge distillation.
\newblock In Jian Su, Kevin Duh, and Xavier Carreras (eds.), \emph{Proceedings
  of the 2016 Conference on Empirical Methods in Natural Language Processing},
  pp.\  1317--1327, Austin, Texas, November 2016. Association for Computational
  Linguistics.
\newblock \doi{10.18653/v1/D16-1139}.
\newblock URL \url{https://aclanthology.org/D16-1139/}.

\bibitem[{Kimi Team} et~al.(2026){Kimi Team}, Bai, Bai, Bao, C., Cai, Cai, Cao,
  Cao, Chai, Charles, Che, Chen, Chen, Chen, Chen, Chen, Chen, Chen, Chen,
  et~al.]{mopd3}
{Kimi Team}, Tongtong Bai, Yifan Bai, Yiping Bao, M.~C., Jianfeng Cai, Xinyuan
  Cai, Peizhou Cao, Yuxuan Cao, Ziwei Chai, Y.~Charles, H.~S. Che, Guanduo
  Chen, Guangyu Chen, Guanzheng Chen, Huarong Chen, Jia Chen, Jianlong Chen,
  Jun Chen, Kexin Chen, et~al.
\newblock Kimi k3: Open frontier intelligence, 2026.
\newblock URL \url{https://arxiv.org/abs/2607.24653}.

\bibitem[{King's College London}(2026)]{kcl2026create}
{King's College London}.
\newblock {King's Computational Research, Engineering and Technology
  Environment (CREATE)}, 2026.
\newblock URL \url{https://doi.org/10.18742/rnvf-m076}.
\newblock Retrieved October 12, 2026.

\bibitem[Lanham et~al.(2023)Lanham, Chen, Radhakrishnan, Steiner, Denison,
  Hernandez, Li, Durmus, Hubinger, Kernion, Luko{\v{s}}i{\=u}t{\.e}, Nguyen,
  Cheng, Joseph, Schiefer, Rausch, Larson, McCandlish, Kundu, Kadavath, Yang,
  Henighan, Maxwell, Telleen-Lawton, Hume, Hatfield-Dodds, Kaplan, Brauner,
  Bowman, and Perez]{lanham2023measuring}
Tamera Lanham, Anna Chen, Ansh Radhakrishnan, Benoit Steiner, Carson Denison,
  Danny Hernandez, Dustin Li, Esin Durmus, Evan Hubinger, Jackson Kernion,
  Kamil{\.e} Luko{\v{s}}i{\=u}t{\.e}, Karina Nguyen, Newton Cheng, Nicholas
  Joseph, Nicholas Schiefer, Oliver Rausch, Robin Larson, Sam McCandlish,
  Sandipan Kundu, Saurav Kadavath, Shannon Yang, Thomas Henighan, Timothy
  Maxwell, Timothy Telleen-Lawton, Tristan Hume, Zac Hatfield-Dodds, Jared
  Kaplan, Jan Brauner, Samuel~R. Bowman, and Ethan Perez.
\newblock Measuring faithfulness in chain-of-thought reasoning, 2023.
\newblock URL \url{https://arxiv.org/abs/2307.13702}.

\bibitem[Li et~al.(2026)Li, Wang, Liang, Tian, Fu, and Vasconcelos]{uopsd}
Yijiang Li, Bingyang Wang, Yijun Liang, Yunjie Tian, Di~Fu, and Nuno
  Vasconcelos.
\newblock On-policy self-distillation without any supervision, 2026.
\newblock URL \url{https://arxiv.org/abs/2608.06296}.

\bibitem[Liu et~al.(2026)Liu, Zhang, Li, Lyu, and
  Shang]{liu2026herohindsightenhancedreflectionenvironment}
Haoran Liu, Yuwei Zhang, Xiyao Li, Bohan Lyu, and Jingbo Shang.
\newblock Hero: Hindsight-enhanced reflection from environment observations for
  agentic self-distillation, 2026.
\newblock URL \url{https://arxiv.org/abs/2606.11559}.

\bibitem[Lu \& {Thinking Machines Lab}(2025)Lu and {Thinking Machines
  Lab}]{opd4}
Kevin Lu and {Thinking Machines Lab}.
\newblock On-policy distillation.
\newblock \emph{Thinking Machines Lab: Connectionism}, 2025.
\newblock \doi{10.64434/tml.20251026}.
\newblock https://thinkingmachines.ai/blog/on-policy-distillation.

\bibitem[Olmo et~al.(2025)Olmo, Ettinger, Bertsch, Kuehl, Graham, Heineman,
  Groeneveld, Brahman, Timbers, Ivison, et~al.]{olmo2025olmo}
Team Olmo, Allyson Ettinger, Amanda Bertsch, Bailey Kuehl, David Graham, David
  Heineman, Dirk Groeneveld, Faeze Brahman, Finbarr Timbers, Hamish Ivison,
  et~al.
\newblock Olmo 3.
\newblock \emph{arXiv preprint arXiv:2512.13961}, 2025.

\bibitem[Sang et~al.(2026)Sang, Xu, Zhou, He, Wang, and
  Sun]{sang2026crispcompressedreasoningiterative}
Hejian Sang, Yuanda Xu, Zhengze Zhou, Ran He, Zhipeng Wang, and Jiachen Sun.
\newblock Crisp: Compressed reasoning via iterative self-policy distillation,
  2026.
\newblock URL \url{https://arxiv.org/abs/2603.05433}.

\bibitem[Shao et~al.(2024)Shao, Wang, Zhu, Xu, Song, Bi, Zhang, Zhang, Li, Wu,
  and Guo]{grpo}
Zhihong Shao, Peiyi Wang, Qihao Zhu, Runxin Xu, Junxiao Song, Xiao Bi, Haowei
  Zhang, Mingchuan Zhang, Y.~K. Li, Y.~Wu, and Daya Guo.
\newblock Deepseekmath: Pushing the limits of mathematical reasoning in open
  language models, 2024.
\newblock URL \url{https://arxiv.org/abs/2402.03300}.

\bibitem[Shen et~al.(2026{\natexlab{a}})Shen, Cheng, Zhao, Huang, Li, Zhao, and
  Yu]{shen2026anti}
Guobin Shen, Xiang Cheng, Chenxiao Zhao, Lei Huang, Jindong Li, Dongcheng Zhao,
  and Xing Yu.
\newblock Anti-self-distillation for reasoning rl via pointwise mutual
  information.
\newblock \emph{arXiv preprint arXiv:2605.11609}, 2026{\natexlab{a}}.

\bibitem[Shen et~al.(2026{\natexlab{b}})Shen, Tong, Yan, Shen, Chen, Ye, Hu,
  Miao, Wang, Zhao, Chen, and Ye]{purified}
Zhanming Shen, Jintao Tong, Shaotian Yan, Chen Shen, Hao Chen, Wentao Ye,
  Xiaomeng Hu, Rui Miao, Haobo Wang, Junbo Zhao, Gang Chen, and Jieping Ye.
\newblock Purified opsd: On-policy self-distillation without losing how to
  think, 2026{\natexlab{b}}.
\newblock URL \url{https://arxiv.org/abs/2607.02234}.

\bibitem[Wang et~al.(2026{\natexlab{a}})Wang, Wang, Xiao, Zhou, Pan, Wang, Xu,
  Wen, Ruan, Chen, and
  Qi]{wang2026skillsdskillconditionedselfdistillationmultiturn}
Hao Wang, Guozhi Wang, Han Xiao, Yufeng Zhou, Yue Pan, Jichao Wang, Ke~Xu,
  Yafei Wen, Xiaohu Ruan, Xiaoxin Chen, and Honggang Qi.
\newblock Skill-sd: Skill-conditioned self-distillation for multi-turn llm
  agents, 2026{\natexlab{a}}.
\newblock URL \url{https://arxiv.org/abs/2604.10674}.

\bibitem[Wang et~al.(2026{\natexlab{b}})Wang, Chen, Jin, Wang, and
  Yang]{wang2026openclawrl}
Yinjie Wang, Xuyang Chen, Xiaolong Jin, Mengdi Wang, and Ling Yang.
\newblock Openclaw-rl: Train any agent simply by talking.
\newblock \emph{arXiv preprint arXiv:2603.10165}, 2026{\natexlab{b}}.

\bibitem[Xu et~al.(2024)Xu, Li, Tao, Shen, Cheng, Li, Xu, Tao, and
  Zhou]{xu2024surveyknowledgedistillationlarge}
Xiaohan Xu, Ming Li, Chongyang Tao, Tao Shen, Reynold Cheng, Jinyang Li, Can
  Xu, Dacheng Tao, and Tianyi Zhou.
\newblock A survey on knowledge distillation of large language models, 2024.
\newblock URL \url{https://arxiv.org/abs/2402.13116}.

\bibitem[Yang et~al.(2025)Yang, Li, Yang, Zhang, Hui, Zheng, Yu, Gao, Huang,
  Lv, Zheng, Liu, Zhou, Huang, Hu, Ge, Wei, Lin, Tang, Yang, Tu, Zhang, Yang,
  Yang, Zhou, Zhou, Lin, Dang, Bao, Yang, Yu, Deng, Li, Xue, Li, Zhang, Wang,
  Zhu, Men, Gao, Liu, Luo, Li, Tang, Yin, Ren, Wang, Zhang, Ren, Fan, Su,
  Zhang, Zhang, Wan, Liu, Wang, Cui, Zhang, Zhou, and Qiu]{opd3}
An~Yang, Anfeng Li, Baosong Yang, Beichen Zhang, Binyuan Hui, Bo~Zheng, Bowen
  Yu, Chang Gao, Chengen Huang, Chenxu Lv, Chujie Zheng, Dayiheng Liu, Fan
  Zhou, Fei Huang, Feng Hu, Hao Ge, Haoran Wei, Huan Lin, Jialong Tang, Jian
  Yang, Jianhong Tu, Jianwei Zhang, Jianxin Yang, Jiaxi Yang, Jing Zhou,
  Jingren Zhou, Junyang Lin, Kai Dang, Keqin Bao, Kexin Yang, Le~Yu, Lianghao
  Deng, Mei Li, Mingfeng Xue, Mingze Li, Pei Zhang, Peng Wang, Qin Zhu, Rui
  Men, Ruize Gao, Shixuan Liu, Shuang Luo, Tianhao Li, Tianyi Tang, Wenbiao
  Yin, Xingzhang Ren, Xinyu Wang, Xinyu Zhang, Xuancheng Ren, Yang Fan, Yang
  Su, Yichang Zhang, Yinger Zhang, Yu~Wan, Yuqiong Liu, Zekun Wang, Zeyu Cui,
  Zhenru Zhang, Zhipeng Zhou, and Zihan Qiu.
\newblock Qwen3 technical report, 2025.
\newblock URL \url{https://arxiv.org/abs/2505.09388}.

\bibitem[Yang et~al.(2026)Yang, Wu, Lu, Shen, Zhang, Feng, Zhang, Luo, Lian,
  Wen, and Tao]{yang2026opidonpolicyskilldistillation}
Shuo Yang, Jinyang Wu, Zhengxi Lu, Yuhao Shen, Fan Zhang, Lang Feng, Shuai
  Zhang, Haoran Luo, Zheng Lian, Zhengqi Wen, and Jianhua Tao.
\newblock Opid: On-policy skill distillation for agentic reinforcement
  learning, 2026.
\newblock URL \url{https://arxiv.org/abs/2606.26790}.

\bibitem[Ye et~al.(2026)Ye, Dong, Wu, Huang, and
  Wei]{ye2026onpolicycontextdistillationlanguage}
Tianzhu Ye, Li~Dong, Xun Wu, Shaohan Huang, and Furu Wei.
\newblock On-policy context distillation for language models, 2026.
\newblock URL \url{https://arxiv.org/abs/2602.12275}.

\bibitem[Yu et~al.(2025)Yu, Zhang, Zhu, Yuan, Zuo, YuYue, Dai, Fan, Liu, Liu,
  Liu, Liu, Lin, Lin, Ma, Sheng, Tong, Zhang, Zhang, Zhang, Zhang, Zhu, Zhu,
  Chen, Chen, Wang, Yu, Song, Wei, Zhou, Liu, Ma, Zhang, Yan, Wu, and
  Wang]{yu2025dapo}
Qiying Yu, Zheng Zhang, Ruofei Zhu, Yufeng Yuan, Xiaochen Zuo, YuYue, Weinan
  Dai, Tiantian Fan, Gaohong Liu, Juncai Liu, LingJun Liu, Xin Liu, Haibin Lin,
  Zhiqi Lin, Bole Ma, Guangming Sheng, Yuxuan Tong, Chi Zhang, Mofan Zhang,
  Ru~Zhang, Wang Zhang, Hang Zhu, Jinhua Zhu, Jiaze Chen, Jiangjie Chen,
  Chengyi Wang, Hongli Yu, Yuxuan Song, Xiangpeng Wei, Hao Zhou, Jingjing Liu,
  Wei-Ying Ma, Ya-Qin Zhang, Lin Yan, Yonghui Wu, and Mingxuan Wang.
\newblock {DAPO}: An open-source {LLM} reinforcement learning system at scale.
\newblock In \emph{The Thirty-ninth Annual Conference on Neural Information
  Processing Systems}, 2025.
\newblock URL \url{https://openreview.net/forum?id=2a36EMSSTp}.

\bibitem[Yu et~al.(2018)Yu, Zhang, Yang, Yasunaga, Wang, Li, Ma, Li, Yao,
  Roman, Zhang, and Radev]{yu-etal-2018-spider}
Tao Yu, Rui Zhang, Kai Yang, Michihiro Yasunaga, Dongxu Wang, Zifan Li, James
  Ma, Irene Li, Qingning Yao, Shanelle Roman, Zilin Zhang, and Dragomir Radev.
\newblock {S}pider: A large-scale human-labeled dataset for complex and
  cross-domain semantic parsing and text-to-{SQL} task.
\newblock In \emph{Proceedings of the 2018 Conference on Empirical Methods in
  Natural Language Processing}. Association for Computational Linguistics,
  2018.
\newblock URL \url{https://aclanthology.org/D18-1425/}.

\bibitem[Zhang et~al.(2026)Zhang, Li, Yu, Lu, Jin, Dong, Liu, Hong, Li, Shi,
  Yin, and Shang]{zhang2026learningraresuccessrich}
Yuwei Zhang, Sha Li, Changlong Yu, Qin Lu, Shuowei Jin, Chengyu Dong, Haoran
  Liu, Ilgee Hong, Xintong Li, Zhenyu Shi, Bing Yin, and Jingbo Shang.
\newblock Learning with rare success but rich feedback via reflection-enhanced
  self-distillation, 2026.
\newblock URL \url{https://arxiv.org/abs/2605.12741}.

\bibitem[Zhao et~al.(2026{\natexlab{a}})Zhao, Xie, Liu, Huang, Pang, Chen, and
  Grover]{opsd}
Siyan Zhao, Zhihui Xie, Mengchen Liu, Jing Huang, Guan Pang, Feiyu Chen, and
  Aditya Grover.
\newblock Self-distilled reasoner: On-policy self-distillation for large
  language models.
\newblock In \emph{Forty-third International Conference on Machine Learning},
  2026{\natexlab{a}}.
\newblock URL \url{https://openreview.net/forum?id=Jpxfof0EaS}.

\bibitem[Zhao et~al.(2026{\natexlab{b}})Zhao, Zhang, Xu, Wang, Caiyue, Zhao,
  and Li]{psopsd}
Xuyang Zhao, Liting Zhang, Zichen Xu, Zhihu Wang, Xu~Caiyue, Shiwan Zhao, and
  Qicheng Li.
\newblock Is more privileged information better? from solution traces to
  problem-solving structure in self-distilled reasoning, 2026{\natexlab{b}}.
\newblock URL \url{https://arxiv.org/abs/2608.01589}.

\bibitem[Zhao et~al.(2026{\natexlab{c}})Zhao, Ma, Yang, Feng, Shi, He, Xin,
  Ren, and Wu]{ropsd}
Ziqi Zhao, Xinyu Ma, Liu Yang, Yujie Feng, Daiting Shi, Jingzhou He, Xin Xin,
  Zhaochun Ren, and Xiao-Ming Wu.
\newblock Rosd: Reflective on-policy self-distillation for language model
  reasoning across domains, 2026{\natexlab{c}}.
\newblock URL \url{https://arxiv.org/abs/2605.28014}.

\bibitem[Zheng et~al.(2025)Zheng, Liu, Li, Chen, Yu, Gao, Dang, Liu, Men, Yang,
  Zhou, and Lin]{zheng2025group}
Chujie Zheng, Shixuan Liu, Mingze Li, Xiong-Hui Chen, Bowen Yu, Chang Gao, Kai
  Dang, Yuqiong Liu, Rui Men, An~Yang, Jingren Zhou, and Junyang Lin.
\newblock Group sequence policy optimization, 2025.
\newblock URL \url{https://arxiv.org/abs/2507.18071}.

\end{thebibliography}
